\documentclass{article}

\PassOptionsToPackage{numbers}{natbib}

\usepackage[final]{style}

\usepackage[utf8]{inputenc}
\usepackage[T1]{fontenc}
\usepackage{hyperref}
\usepackage{url}     
\usepackage{booktabs} 
\usepackage{amsfonts}   
\usepackage{nicefrac}   
\usepackage{microtype}  
\usepackage{xcolor}     

\usepackage{algorithm}
\usepackage{algorithmic}
\usepackage{graphicx}
\usepackage{subcaption}
\usepackage{listings}
\usepackage{amsmath}
\usepackage{amssymb}
\usepackage{amsthm}
\usepackage{tcolorbox}

\newcommand{\RR}{\mathbb{R}}
\newcommand{\Qstar}{$Q^*$}
\newcommand{\Rmax}{r_{\mathrm{max}}}
\newcommand{\EE}{\mathbb{E}}

\newcommand{\zero}{\mathbf{0}}

\newtheorem{prop}{Proposition}
\newtheorem{defn}{Definition}

\usepackage{thmtools} 
\usepackage{thm-restate}

\title{Control-Oriented Model-Based Reinforcement Learning with Implicit Differentiation}

\author{
	\parbox{\linewidth}{
		\centering
		Evgenii Nikishin$^{1}$\quad
		Romina Abachi$^{2}$\quad
		Rishabh Agarwal$^{1,3}$\quad
		Pierre-Luc Bacon$^{1,4}$
	}\\
	~\\
	\parbox{\linewidth}{
		\centering
		$^1$Mila, Université de Montréal, 
		$^2$Vector Institute, University of Toronto \\ 
		$^3$Google Research, Brain Team, 
		$^4$Facebook CIFAR AI Chair
	}
}

\begin{document}

\maketitle

\begin{abstract}
  The shortcomings of maximum likelihood estimation in the context of \mbox{model-based} reinforcement learning have been highlighted by an increasing number of papers. When the model class is misspecified or has a limited representational capacity, model parameters with high likelihood might not necessarily result in high performance of the agent on a downstream control task. To alleviate this problem, we propose an end-to-end approach for model learning which directly optimizes the expected returns using implicit differentiation. We treat a value function that satisfies the Bellman optimality operator induced by the model as an implicit function of model parameters and show how to differentiate the function. We provide theoretical and empirical evidence highlighting the benefits of our approach in the model misspecification regime compared to likelihood-based methods.
\end{abstract}

\section{Introduction} 

The conceptual separation between model learning and policy optimization is the basis for much of the work on model-based reinforcement learning~(MBRL)~\citep{Grefenstette1990, sutton1991dyna, Lin1992, Boots2011, chua2018deep, hafner2019learning, janner2019trust, kaiser2019model}. 
A standard MBRL agent first estimates the transition parameters and the reward function of a Markov Decision Process and then uses the approximate model for planning~\citep{Theil1957, kurano1972discrete, Mandl1974, Georgin1978, Borkar1979, HernndezLerma1985, Manfred1987, Sato1988}. 
If the estimated model perfectly captures the actual system, the resulting policies are not affected by the model approximation error. However, if the model is imperfect, the inaccuracies can lead to nuanced effects on the policy performance~\citep{Abbad1992, Manfred1987}. 
Several works~\citep{skelton1989model, joseph2013reinforcement, lambert2020objective} have pointed out on \emph{the objective mismatch} in MBRL and demonstrated that optimization of model likelihood might be unrelated to optimization of the returns achieved by the agent that uses the model.
For example, accurately predicting individual pixels of the next state~\citep{kaiser2019model} might be neither easy nor necessary for decision making.
Motivated by these observations, our paper studies control-oriented model learning that takes into account how the model is used by the agent. 

While much of the work in control-oriented model learning has focused on robust or uncertainty-based methods~\citep{Ludwig1982, Nilim2005, Iyengar2005, Xu2010, yu2020mopo}, we propose an algorithm for learning a model that directly optimizes the expected return using implicit differentiation~\citep{Christianson1994, Griewank2008}. Specifically, we assume that there exists \emph{an implicit function} that takes the model as input and outputs a value function that is a fixed point of the Bellman optimality operator~\citep{Denardo1967} induced by the model. 
We then calculate the derivatives of the optimal value function with respect to the model parameters using the implicit function theorem (IFT), allowing us to form a differentiable computational graph from model parameters to the sum of rewards.
In reference to~\citep{rust1988maximum,Sorg2010,baconlagrangian}, we call our control-oriented method \emph{optimal model design}~(OMD).

Our contributions can be summarized as follows:
\begin{itemize}
    \item We propose OMD, an end-to-end MBRL method that optimizes expected returns directly.
    \item We characterize the set of OMD models in the tabular case and derive an approximation bound on the optimal value function that is tighter than the likelihood-based bound.
    \item We propose a series of approximations to scale our approach to non-tabular environments.
    \item We demonstrate that OMD outperforms likelihood-based MBRL agents under the model misspecification in both tabular and non-tabular settings. This finding suggests that our method should be preferred when we cannot approximate the true model accurately.
    \item We empirically demonstrate that models obtained by OMD can have lower likelihood than a random model yet generate useful targets for updating the value function. This finding suggests that likelihood optimization might be an unnecessary step for MBRL.
\end{itemize}


\begin{figure}[t]
\begin{center}
\centerline{\includegraphics[width=0.7\columnwidth]{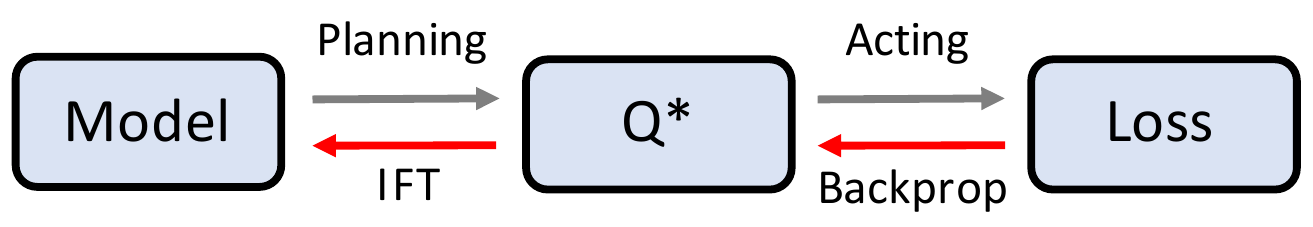}}
\caption{Illustration of the \emph{Optimal Model Design} approach: we treat the optimal Q-function as an implicit function of the model and calculate the gradient with respect to model parameters via the implicit function theorem as described in Section~\ref{sec:ift}.}
\label{fig:comp_graph}
\end{center}
\vskip -0.2in
\end{figure}

\section{Related work}
\label{sec:related}

\paragraph{Learning control-oriented models.} 
Earlier work in optimal control and econometrics \citep{skelton1989model, rust1988maximum} studied the relation between the model approximation error and the control performance and noted that true parameter identification could be suboptimal when the model class is limited.
\citet{joseph2013reinforcement} were one of the first to address the objective mismatch \citep{lambert2020objective} and proposed an algorithm for training models that maximize the expected returns using zero-order optimization. 

Several papers have proposed model learning approaches that optimize other return-aware objectives.
\citet{farahmand2017value} train a model to minimize the difference between values of the real next states and the next states predicted by the dynamics. \citet{abachi2020policy} use the norm of the difference between policy gradients as the model objective.
\citet{d2020gradient} use a weighted maximum likelihood objective where the weights are chosen to minimize the difference between the true policy gradient and the policy gradient in the MDP induced by the model. 
\citet{schrittwieser2019mastering} use tree search and train a model for image-based states by encoding them into a latent space and predicting a reward, a policy, and values without reconstructing the images. 

The idea of differentiable planning has also been investigated. 
\citet{amos2018differentiable} learn a model via differentiating the Karush–Kuhn–Tucker conditions in the LQR setting \cite{dorato1994linear}.
\citet{tamar2016value} uses a differentiable approximation of the value iteration algorithm to learn a planner.
\citet{amos2020differentiable} optimize the parameters of a sampling distribution in Cross-Entropy Method \cite{rubinstein1997optimization} using a differentiable approximation of Top-K operation. 

Several works have theoretically studied the control-oriented model learning.
\citet{ayoub2020model} derive regret bounds for models used to predict values. 
\citet{grimm2020value} introduce the principle of value equivalence for MBRL defining two models to be equivalent if they induce the same Bellman operator. 

Our work is closely related to the above papers but proposes to learn models by directly optimizing the sum of rewards in an end-to-end manner via gradient-based methods.

\paragraph{Implicit function theorem.} 
Implicit differentiation has been applied for a variety of bi-level optimization problems. 
\citet{lorraine2020optimizing} treat weights of a neural network as an implicit function of hyperparameters and use IFT to optimize the hyperparameters. 
\citet{rajeswaran2019meta} study meta-learning and apply IFT to compute the outer loop gradient without the need to differentiate through the inner loop iterations. 
Instead of treating a neural network as a sequence of layers that transform an input, \citet{bai2019deep} propose an implicit layer that corresponds to an infinite depth neural network and find a fixed point of the layer via IFT. 
Our method also solves a bi-level problem: in the inner loop, we train an action-value function compatible with the model, while in the outer loop we update the model parameters towards maximizing the expected returns.

\section{Preliminaries}
\label{sec:back}

Reinforcement Learning (RL)~\citep{sutton2018reinforcement} methods follow the Markov Decision Process (MDP) formalism. 
An MDP is defined as $\mathcal{M} = (\mathcal{S}, \mathcal{A}, \gamma, p, r, \rho_0)$, where $\mathcal{S}$ is a state space, $\mathcal{A}$ is an action space, $p(s'|s,a)$ is a transition probability distribution (often called \emph{dynamics}), $r(s,a)$ is a reward function, $\gamma \in [0, 1)$ is a discount factor, and $\rho_0(s)$ is an initial state distribution. 
The pair ($p$, $r$) is jointly called \emph{the true model}.
The goal of an agent is to learn a policy $\pi(a|s)$ that maximizes the expected discounted sum of rewards $J(\pi) = \mathbb{E}_\pi\left[\sum_{t=0}^\infty \gamma^t r(s_t, a_t)\right]$. 
The performance of the agent following the policy $\pi$ can also be quantified using the action value function $Q^\pi(s,a) = \mathbb{E}_\pi\left[\sum_{t=0}^\infty \gamma^t r(s_t, a_t) | s_0 = s, a_0 = a\right]$. 

Model-based RL algorithms typically train a model ($p_\theta$, $r_\theta$) and use it for policy or value learning. 
Traditional methods based on Dyna \cite{sutton1991dyna} rely on maximum likelihood estimation (MLE) of model parameters $\theta$. 
For example, if the true model is assumed to be Gaussian with a parameterized mean and a fixed variance, maximizing the likelihood is equivalent to minimizing the mean squared error of the prediction, namely, to solving
\begin{equation}
    \min_\theta \EE_{s, a, s'}\left[\|f_\theta(s,a) - s'\|^2\right], \quad
    \min_\theta \EE_{s, a, r}\left[(r_\theta(s,a) - r)^2\right].
\end{equation}

\section{Optimal Model Design for Tabular MDPs}
\label{sec:omd_tab}

Consider a modification of the original RL problem statement, which was first proposed by \citet{rust1988maximum} and revisited by \citet{baconlagrangian}. In addition to maximizing the expected returns $J$, we introduce a constraint forcing the action value function $Q$ to satisfy the Bellman equation induced by the model. The optimization problem becomes
\begin{equation}
    \label{eq:constr}
    \max_{Q, \theta} J(\pi_Q) \qquad
    \text{s.t. } Q(s, a) = B^{\theta} Q(s, a) \enspace \forall s \in \mathcal{S}, a \in \mathcal{A}.
\end{equation}

$B^{\theta}$ is \emph{the soft Bellman optimality operator} with respect to the model and $\pi_Q$ is the softmax policy:
\begin{equation}
    \label{eq:bellman_op_model}
    B^{\theta} Q(s, a) \triangleq r_\theta(s, a) + \gamma \mathbb{E}_{p_\theta(s'|s, a)}\log\sum_{a'}\exp Q(s', a'), \qquad
    \pi_Q(a|s) = \frac{\exp Q(s,a)}{\sum_{a'} \exp Q(s,a')}.
\end{equation}

We choose the soft Bellman operator with $\log\sum_{a'}\exp Q(s', a')$ over the ``hard'' version with $\max_{a'} Q(s', a')$ because of the differentiability of log-sum-exp. 
We also use a temperature $\alpha$ in softmax and log-sum-exp but omit it from the expressions for simplicity.
Note that finding a fixed point of the soft Bellman optimality operator corresponds to solving the MaxEnt RL formulation \cite{ziebart2008maximum}, but for a sufficiently small value of $\alpha$, the difference is negligible.\footnote{More details about the soft Bellman operator and MaxEnt RL could be found in \citep{levine2018reinforcement}.}

Suppose there exists an \emph{implicit} function $\varphi(\theta) = Q^*$ that takes as input a model and outputs a Q-function that satisfies the constraint in \eqref{eq:constr}. 
The sequence of transformations from the model parameters to the agent's performance can be described then using the following graph:
\begin{equation}
    \label{eq:graph}
    \theta \xrightarrow[]{\varphi} Q^* \xrightarrow[]{\exp} \pi_{Q^*} \xrightarrow[]{\text{act}} J.
\end{equation}
In Section~\ref{sec:ift}, we show how $\frac{\partial \varphi(\theta)}{\partial \theta}$ can be calculated using the implicit function theorem (IFT). Since $\frac{\partial J(\pi)}{\partial \pi}$ can be calculated using the policy gradient theorem \cite{sutton1999policy}, we can apply automatic differentiation to calculate the gradient with respect to $\theta$:
\begin{equation}
    \label{eq:grad_params}
    \frac{\partial J(\theta)}{\partial \theta} 
    = \underbrace{\frac{\partial J(\pi)}{\partial \pi}}_{\text{PG}} \cdot
    \underbrace{\frac{\partial \pi(Q^*)}{\partial Q^*}}_{\text{softmax}} \cdot
    \underbrace{\frac{\partial \varphi(\theta)}{\partial \theta}}_{\text{IFT}}.
\end{equation}

Given the expression for the gradient of $J$ with respect to $\theta$, we use an appropriate optimization method to train the model. 
We call the approach \emph{optimal model design} (OMD). 
Note that Dyna-based methods also train the Q-function to satisfy the constraint in \eqref{eq:constr} while using the likelihood as the objective for model parameters $\theta$ \citep{rajeswaran2020game}. 
In contrast, we train $\theta$ to \emph{directly optimize} the expected returns.

The optimization problem \eqref{eq:constr} suggests that OMD is a policy-based method~\citep{sutton1999policy}. However, we can turn it into a value-based approach~\citep{watkins1992q} by replacing the objective $J(\pi_Q)$ with the Bellman error:
\begin{equation}
    \label{eq:bellman_outer}
    \min_{Q, \theta} L^{\text{true}}(Q) \triangleq \sum_{s,a} \left(Q(s,a) - BQ(s,a)\right)^2 \qquad
    \text{s.t. } Q(s, a) = B^{\theta} Q(s, a) \enspace \forall s \in \mathcal{S}, a \in \mathcal{A},
\end{equation}

where $B$, similarly to $B^\theta$, is the soft Bellman operator but induced by \emph{the true reward $r$ and dynamics $p$}. We discuss the relation between the models obtained by solving problems~\eqref{eq:constr} and \eqref{eq:bellman_outer} in Section~\ref{sec:analysis}.

While the constraint $Q(s, a) = B^{\theta} Q(s, a)$ has to be satisfied for all state-action pairs limiting the approach to tabular MDPs, we show an extention to the function approximation case in Section~\ref{sec:omd_cont}.

\subsection{Implicit Differentiation}
\label{sec:ift}

In this subsection, we state the implicit function theorem used to calculate $\frac{\partial \varphi(\theta)}{\partial \theta}$.

\begin{restatable}{theorem}{ift}\emph{(Cauchy, Implicit Function)}
    Let $f: \Theta \times \mathcal{W} \to \mathcal{W}$ be a continuously differentiable function and $(\tilde{\theta}, \tilde{w})$ be a point satisfying $f(\tilde{\theta}, \tilde{w}) = \zero$. If the Jacobian $\frac{\partial f(\tilde{\theta}, \tilde{w})}{\partial w}$ is invertible, then there exists an open set $U \subseteq \Theta$ containing $\tilde{\theta}$ and a unique continuously differentiable function $\varphi$ such that $\varphi(\tilde{\theta}) = \tilde{w}$ and $f(\theta, \varphi(\theta)) = \zero$ for all $\theta \in U$. Moreover, \\
    \begin{equation}
        \label{eq:ift}
        \frac{\partial\varphi(\theta)}{\partial\theta} = -\left(\frac{\partial f(\theta, w^*)}{\partial w}\right)^{-1} \cdot \frac{\partial f(\theta, w^*)}{\partial \theta}\Big\rvert_{w^* = \varphi(\theta)} \quad \forall \theta \in U.
    \end{equation}
\end{restatable}

We provide a proof in Appendix \ref{sec:proof}.
Note that \eqref{eq:ift} requires only a final point $w^*$ satisfying the constraint and does not require knowledge about $\varphi$ itself. 
Hence, $\varphi$ can be any black-box function outputting~$w^*$.
The gradient of the scalar objective $J$ or $L^{\text{true}}$ is calculated using \eqref{eq:ift}. 
To use backpropagation, we only need to define the product of a vector and $\frac{\partial\varphi(\theta)}{\partial\theta}$. 
We provide an implementation of a custom vector-Jacobian product for the implicit function $\varphi$ in Appendix \ref{sec:jax_ift} allowing to use $\varphi$ as a block in a differentiable computational graph.

\subsection{Benefits under Model Misspecification}
\label{sec:misspec_tab}

In the previous subsection, we showed how to use implicit differentiation for training a model that aims to maximize the expected returns.
In this subsection, we demonstrate that such a control-oriented model is preferable over a likelihood-based in the setting where the true model is not representable by a chosen parametric class.

Let $r_\theta \in \RR^{|\mathcal{S}| \times |\mathcal{A}|}$ and $p_\theta(s'|s,a) \in \RR^{|\mathcal{S}| \times |\mathcal{S}| \times |\mathcal{A}|}$ be a parametric model, where parameters in $p_\theta$ denote the corresponding logits and each parameter in $r_\theta$ is a reward for a state-action pair.
We consider a set of parameters $\{\theta: \|\theta\| \leq \kappa\}$ with the \emph{bounded norm} and use $\kappa$ as a measure of the model misspecification. 
By decreasing the bound of the norm of $\theta$, we get a more misspecified model class.
To isolate the model learning aspect, we consider the exact RL setting without sampling. 
We take a 2 state, 2 action MDP shown in Figure \ref{fig:equiv_mdps_and_bounds} with a discount factor $\gamma = 0.9$ and a uniform initial distribution $\rho_0$. 
For every $\theta$, a function $\varphi$ outputs the corresponding $Q^*$ via performing the fixed point iteration until convergence. 
$Q^*$ is transformed into the policy $\pi_{Q^*}$ via softmax with the temperature $\alpha=0.01$. 
Given the policy $\pi_{Q^*}$, we calculate $J$ in a closed form~\citep{sutton2018reinforcement}.

For OMD, we obtain the gradient of $J$ with respect to $\theta$ using the expression \eqref{eq:grad_params}. 
We then apply the projected gradient ascent where after each step we make a projection on a space of bounded parameters via clipping $\theta$ to $\frac{\kappa}{\|\theta\|} \theta$ if $\|\theta\| > \kappa$.
Finding an MLE solution corresponds to minimizing the average KL divergence
\begin{equation*}
    \label{eq:kl}
    \overline{\text{D}_\text{KL}}(p||p_\theta) = \frac{1}{|\mathcal{S}| \cdot |\mathcal{A}|}\sum_{s,a,s'} p(s'|s,a) \log \frac{p(s'|s,a)}{p_\theta(s'|s,a)}
\end{equation*}
for optimizing $p_\theta$ and minimizing the squared error for $r_\theta$.
We similarly perform the projected gradient descent and call the agent MLE (despite having the exact setting without estimation).

\begin{figure}[t]
\vskip -0.2in
\centering
\centerline{\includegraphics[width=0.5\columnwidth]{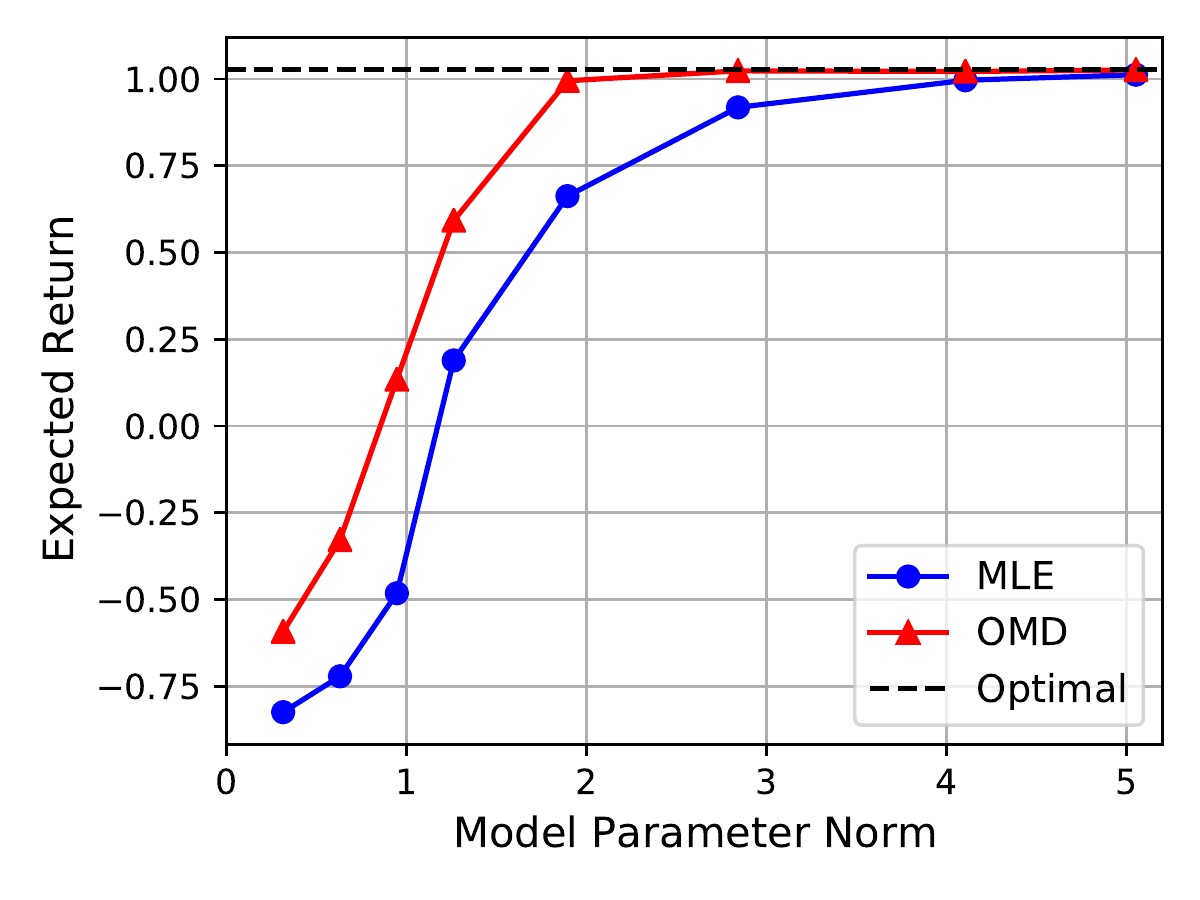}}
\vskip 0.07in
\caption{Expected returns for the tabular MDP under the model class misspecification. The OMD model optimizes the expected returns directly, while the MLE agent minimizes the KL divergence for model learning. OMD outperforms MLE when the model representational capacity is limited.}
\label{fig:misspec_tab}
\vskip -0.17in
\end{figure}
The resulting $J$ as a function of the norm bound $\kappa$ is shown in Figure~\ref{fig:misspec_tab}. 
When the true model is not representable by a chosen class, OMD learns a model that uses its representational capacity for helping the agent to maximize the expected returns, while the MLE agent tries to predict the next states and rewards accurately while discarding the true objective function the agent seeks to optimize.

In MDPs with high-dimensional state spaces~\citep{bellemare2013arcade, beattie2016deepmind} where the underlying dynamics are complex, having a model that will accurately predict the next observation might be expensive and unnecessary for decision making.
Figure~\ref{fig:misspec_tab} reflects the problems an MLE-based model will face for such environments and provides evidence for using control-oriented models that leverage the available capacity of the model more effectively.

\section{Theoretical Analysis}
\label{sec:theory}

In the previous section, we have empirically demonstrated that OMD outperforms Dyna-style~\citep{sutton1991dyna} MBRL agents when the model capacity is limited. 
This section characterizes the set of optimal solutions of OMD and compares the $Q^*$ approximation bounds for OMD and MLE agents.

\subsection{Optimal Solutions for OMD}
\label{sec:analysis}

We use the principle of value equivalence for MBRL~\citep{grimm2020value} and argue that value equivalent models are optimal solutions to \eqref{eq:constr} and \eqref{eq:bellman_outer}. 

\begin{defn}[Optimal value equivalence] Let \Qstar\ be an optimal action-value function for the unconstrained RL problem. The models with parameters $\theta$ and $\theta'$ are \Qstar-equivalent if
\begin{equation}
    B^\theta Q^*(s,a) = B^{\theta'} Q^*(s,a) \enspace \forall s \in \mathcal{S}, a \in \mathcal{A}.
\end{equation}
\end{defn} 

The definition is a slight modification of the value equivalence used in \cite{grimm2020value}: instead of requiring the Bellman operators to be equal for a set of value functions and policies, we require the equality for a chosen \Qstar\ only.
The subset of models that are \Qstar-equivalent forms an \emph{equivalence class}~$\Theta_{Q^*}$.
\begin{prop} If we let the soft Bellman operator~\eqref{eq:bellman_op_model} temperature $\alpha \to 0$ and let $\theta$ be \emph{any} model parameters from the equivalence class $\Theta_{Q^*}$, then $(Q^*, \theta)$ is a solution for \eqref{eq:constr} and \eqref{eq:bellman_outer}.
\end{prop} 

This property holds by construction. The optimal Q-function maximizes the objective in the true MDP. As the log-sum-exp temperature in \eqref{eq:bellman_op_model} approaches 0, we recover the ``hard'' target in the Bellman optimality operator:
\begin{equation}
    \label{eq:lim}
    \lim_{\alpha \to 0} \alpha \log \sum_{a'} \exp \frac{1}{\alpha} Q(s', a') = \max_{a'} Q(s', a').
\end{equation}

Thus, if we set $\theta$ to the true model, $Q^*$ will satisfy the Bellman equation $Q^*(s,a) = B^\theta Q^*(s,a)$.
But even though the true model belongs to the equivalence class $\Theta_{Q^*}$, it is \emph{not identifiable}: all models from $\Theta_{Q^*}$ are going to be indistinguishable for OMD. 
Seemingly undesirable at first glance, it allows OMD choosing \emph{any} model that induces the same Bellman operator, which is beneficial under the model misspecification as shown in Section \ref{sec:misspec_tab}.

We provide an example of a model that is \Qstar-equivalent with the true model in Figure~\ref{fig:equiv_mdps_and_bounds}. 
The model differs significantly, demonstrating that the equivalence class $\Theta_{Q^*}$ consists of multiple elements. 
Moreover, the dynamics learned by OMD are deterministic, suggesting that OMD can choose \emph{a simpler} model that will have the same $Q^*$ as the true model. 
Drawing the connection to the prior work on state abstractions~\citep{li2006towards}, the fact that MDPs have the same optimal action values indicates that the learned models can be seen as $Q^*$-irrelevant with respect to a state abstraction over $\mathcal{S}$.

\begin{figure}[t]
\centering
\begin{subfigure}{0.48\textwidth}
	\centerline{\includegraphics[width=\textwidth]{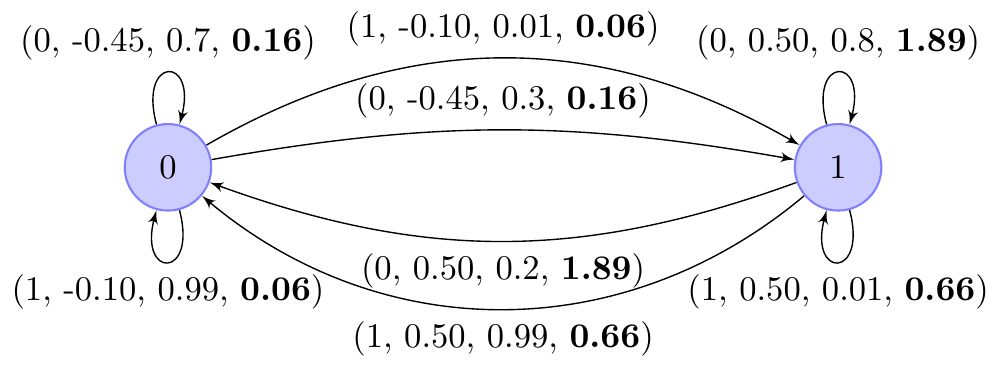}}
    \centerline{\includegraphics[width=\textwidth]{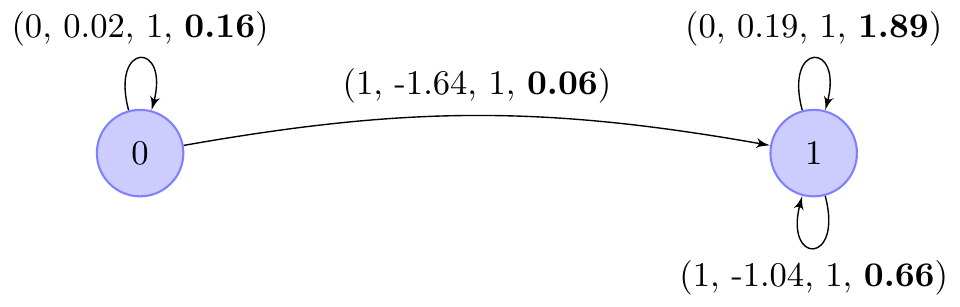}}
\end{subfigure}
~
\begin{subfigure}{0.48\textwidth}
	\centerline{\includegraphics[width=\textwidth]{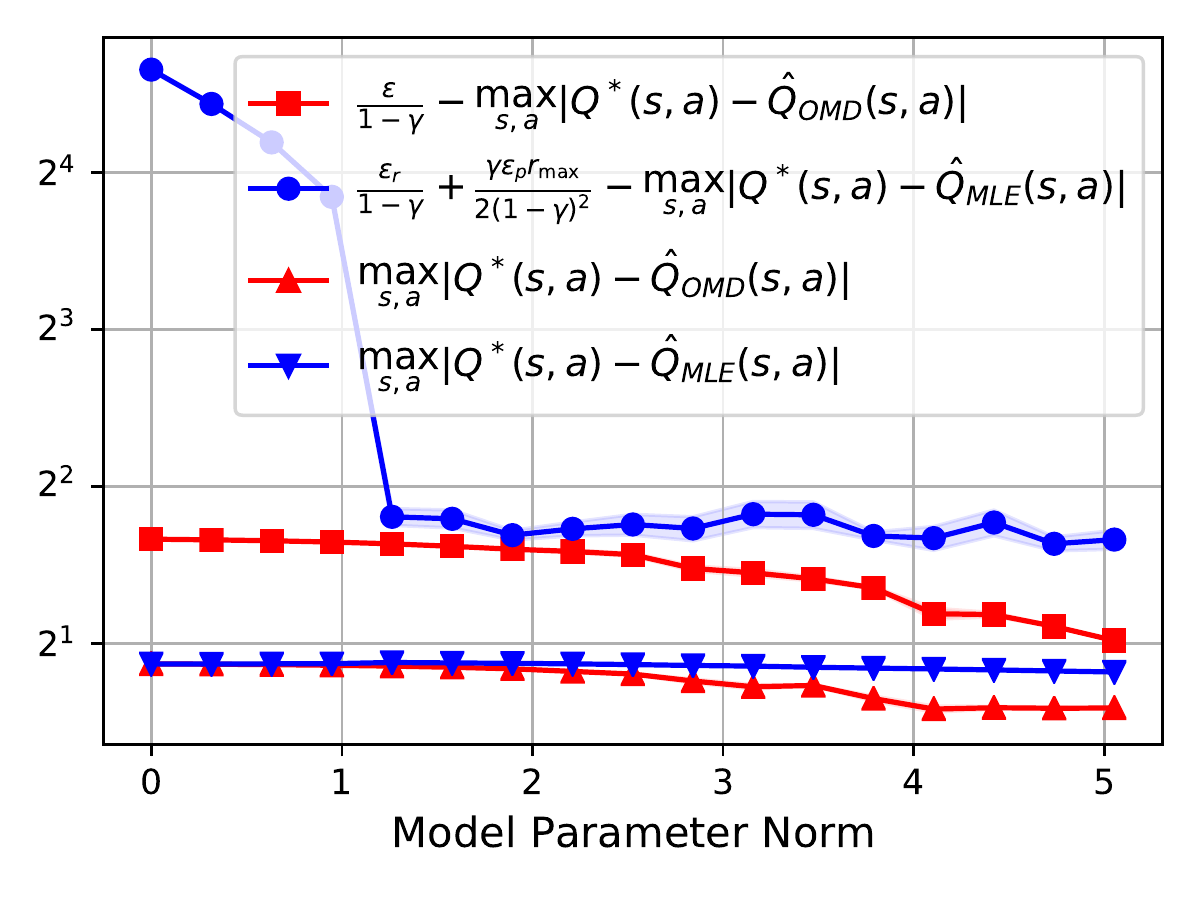}}
\end{subfigure}
\vskip 0.05in
\caption{\textbf{Left:} Two different MDPs with the same optimal Q-function (a fixed point of the induced Bellman operator). Circles represent states, tuples are organized as (action, reward, transition probability, optimal Q value). Top: the original MDP taken from \citep{dadashi2019value}. Bottom: an MDP with a trained OMD model. \textbf{Right:} $Q^*$ approximation error and tightness of the error bounds under the model misspecification. Given a limited model representation capacity, OMD agent approximates $Q^*$ more accurately and enjoys a tighter bound.}
\label{fig:equiv_mdps_and_bounds}
\vskip -0.07in
\end{figure}

\subsection{Approximation Bound}
\label{sec:bounds}
Our next result relates approximation errors for the optimal Q-functions under the OMD and MLE models. For simplicity, we analyze the setting with $\alpha \to 0$ and the Bellman error \eqref{eq:bellman_outer} as the objective. 

\begin{restatable}{theorem}{qstarerror} \emph{(\Qstar\ approximation error)} Let $Q^*$ be the optimal action-value function for the true MDP. Let $\hat{Q}_\mathrm{OMD}$ and $\hat{Q}_\mathrm{MLE}$ be the fixed points of the Bellman optimality operators for approximate OMD and MLE models respectively. Then,
\begin{itemize}
\item If the MLE dynamics $\hat{p}$ and reward $\hat{r}$ have the bounded errors 
$\max_{s, a} \left\|p(\cdot|s,a) - \hat{p}(\cdot|s,a)\right\|_1 = \epsilon_p$ and $\max_{s,a}\left|r(s,a) - \hat{r}(s,a)\right| = \epsilon_r$, 
and the reward function is bounded $r(s,a) \in \left[0, r_{\max}\right] \enspace \forall s,a$, we have
\begin{equation*}
    \max_{s,a} \left|Q^*(s,a) - \hat{Q}_\mathrm{MLE}(s,a)\right| \le \frac{\epsilon_r}{1-\gamma} + \frac{\gamma \epsilon_p \Rmax}{2(1-\gamma)^2};
\end{equation*}
\item If the Bellman optimality operator induced by the OMD model $\hat{\theta}$ has the bounded error 
$\max_{s,a} \left|B\hat{Q}_{\mathrm{OMD}}(s,a) - B^{\hat{\theta}} \hat{Q}_{\mathrm{OMD}}(s,a)\right| = \epsilon$, 
we have
\begin{equation*}
    \max_{s,a} \left|Q^*(s,a) - \hat{Q}_{\mathrm{OMD}}(s,a)\right| \le \frac{\epsilon}{1-\gamma}.
\end{equation*}
\end{itemize}
\end{restatable}

We prove the bounds in Appendix~\ref{sec:bounds_proof} using similar arguments as the proof of the simulation lemma \citep{kearns2002near}. 
The MLE bound has a $\frac{1}{(1-\gamma)^2}$ term making the bound loose compared to the OMD bound with only the $\frac{1}{1-\gamma}$ term.
The bound suggests that OMD approximation error translates into a lower $Q^*$ approximation error than for the MLE model. 
Figure~\ref{fig:equiv_mdps_and_bounds} compares empirically the errors and the tightness of the bounds for a tabular MDP where $Q^*$ can be computed exactly.
The result provides evidence that OMD indeed achieves a lower $Q^*$ approximation error compared to an agent that seeks to estimate $p$ and $r$ accurately.
Motivated by the theoretical findings, the next section discusses a practical version of OMD.

\begin{algorithm}[tb]
   \caption{Model Based RL with Optimal Model Design}
   \label{alg:omd}
\begin{algorithmic}
   \STATE {\bfseries Input:} Initial parameters $w$ and $\theta$, empty replay buffer $\mathcal{D}$.
   \REPEAT
   \STATE Set $s$ to the current state, sample an action $a$ using softmax over $Q_w(s,a)$.
   \STATE Take the action $a$, observe $r = r(s,a), s' \sim p(s'|s,a)$, add $(s,a,s',r)$ to $\mathcal{D}$.
   \FOR{$i=1$ {\bfseries to} $K$}
   \STATE Sample $(s,a)$ from $\mathcal{D}$, apply the model to get $r = r_\theta(s,a), s' \sim p_\theta(s'|s,a)$. 
   \STATE Update $Q_w$ parameters $w$ to minimize $L(\theta, w)$.
   \ENDFOR
   \STATE {\color{blue} Update model parameters $\theta$ according to \eqref{eq:grad_final}.}
   \UNTIL{the maximum number of interactions is reached}
\end{algorithmic}
\end{algorithm}

\section{OMD with Function Approximation} 
\label{sec:omd_cont}

Section \ref{sec:omd_tab} describes optimal model design, a non-likelihood-based method for learning models in tabular MDPs. 
In this section, we propose several approximations to make OMD practically applicable. 
We analyze the effect of the approximations and perform an ablation study in Appendix~\ref{sec:ablation}.

\textbf{Q-network.} We use a neural network with parameters $w$ to approximate the Q-values. The network is trained to minimize the Bellman error induced by the model $\theta$:
\begin{equation}
    \label{eq:bellman_err_model}
    L(\theta, w) \triangleq \mathbb{E}_{s,a} [Q_w(s, a) - B^{\theta} Q_{\bar{w}}(s, a)]^2 \to \min_w,
\end{equation}
where $\bar{w}$ is a target copy of parameters $w$ updated using exponential moving average, a standard practice to increase the stability of deep Q-learning \citep{mnih2015human}. We also use double Q-learning \cite{hasselt2010double, fujimoto2018addressing} but omit it from the equations for simplicity. To estimate the expectation, we use a replay buffer \citep{mnih2015human}. 

\textbf{Constraint.} The constraint in \eqref{eq:constr} and \eqref{eq:bellman_outer} should be satisfied for all state-action pairs making it impractical for non-tabular MDPs. We introduce an alternative but similar constraint, the first-order optimality condition for minimizing the Bellman error~\eqref{eq:bellman_err_model}: $\frac{\partial L(\theta, w)}{\partial w} = \mathbf{0}.$

\textbf{Implicit differentiation.} The process of training $\theta$ is bi-level: in the inner loop, we optimize the Q-function parameters to get optimal $w^*$ corresponding to a fixed model $\theta$; in the outer loop, we make a gradient update of $\theta$. We make $K$ steps of an optimization method to approximate $w^* = \varphi(\theta)$ where $K$ is a hyperparameter and reuse the weights from the previous outer loop iterations. We follow \citet{rajeswaran2020game} and approximate the inverse Jacobian term in $\frac{\partial \varphi(\theta)}{\partial \theta}$ with the identity matrix. Surprisingly, we did not observe benefits when using the inverse Jacobian term. We investigate the phenomenon deeper and discuss possible explanations in Appendix \ref{sec:ift_sens}.

\textbf{Objective.} We consider the problem \eqref{eq:bellman_outer} and use the Bellman error as the outer loop objective:
\begin{equation}
    \label{eq:bellman_err_true}
    L^{\text{true}}(w) \triangleq \mathbb{E}_{s,a} [Q_w(s, a) - B Q_{\bar{w}}(s, a)]^2,
\end{equation}
where $B$, again, is the soft Bellman optimality operator induced by the true reward $r$ and dynamics $p$.

Note that the objective~\eqref{eq:bellman_err_true} is used for estimating the gradient with respect to $\theta$ \emph{only} and $w$ is trained to optimize $L(\theta, w)$. While both $L^{\text{true}}$ and $J$ objectives could be used for training $\theta$, we found that the latter requires more samples to converge. Note that optimizing the $L^{\text{true}}$ still corresponds to maximizing the (entropy-regularized) expected returns.

\textbf{Resulting gradient.} 
The changes above in the objective function and the constraint yield the following optimization problem:
\begin{equation}
    \label{eq:omd_final}
    \min_{w, \theta} L^{\text{true}}(w) \qquad
    \text{s.t. } \frac{\partial L(\theta, w)}{\partial w} = \mathbf{0}.
\end{equation}

The Q-function and IFT approximations and result in the following gradient with respect to the model parameters:
\begin{equation}
    \label{eq:grad_final}
    \frac{\partial L^{\text{true}}(\theta)}{\partial \theta}
    \approx
    - 
    \underbrace{\frac{\partial L^{\text{true}}(w^*)}{\partial w}}_{\text{grad Bellman}} \cdot 
    \underbrace{\frac{\partial^2 L(\theta, w^*)}{\partial \theta \partial w}}_{\text{approx IFT}} \Big\rvert_{w^* = \varphi(\theta)}
\end{equation}

The OMD algorithm is summarised in Algorithm \ref{alg:omd}. The only difference between Dyna-based approaches and OMD (highlighted in blue) is given by the gradient used to train model parameters $\theta$.

\begin{figure}[t]
\centering
\begin{subfigure}{0.48\textwidth}
	\centerline{\includegraphics[width=\textwidth]{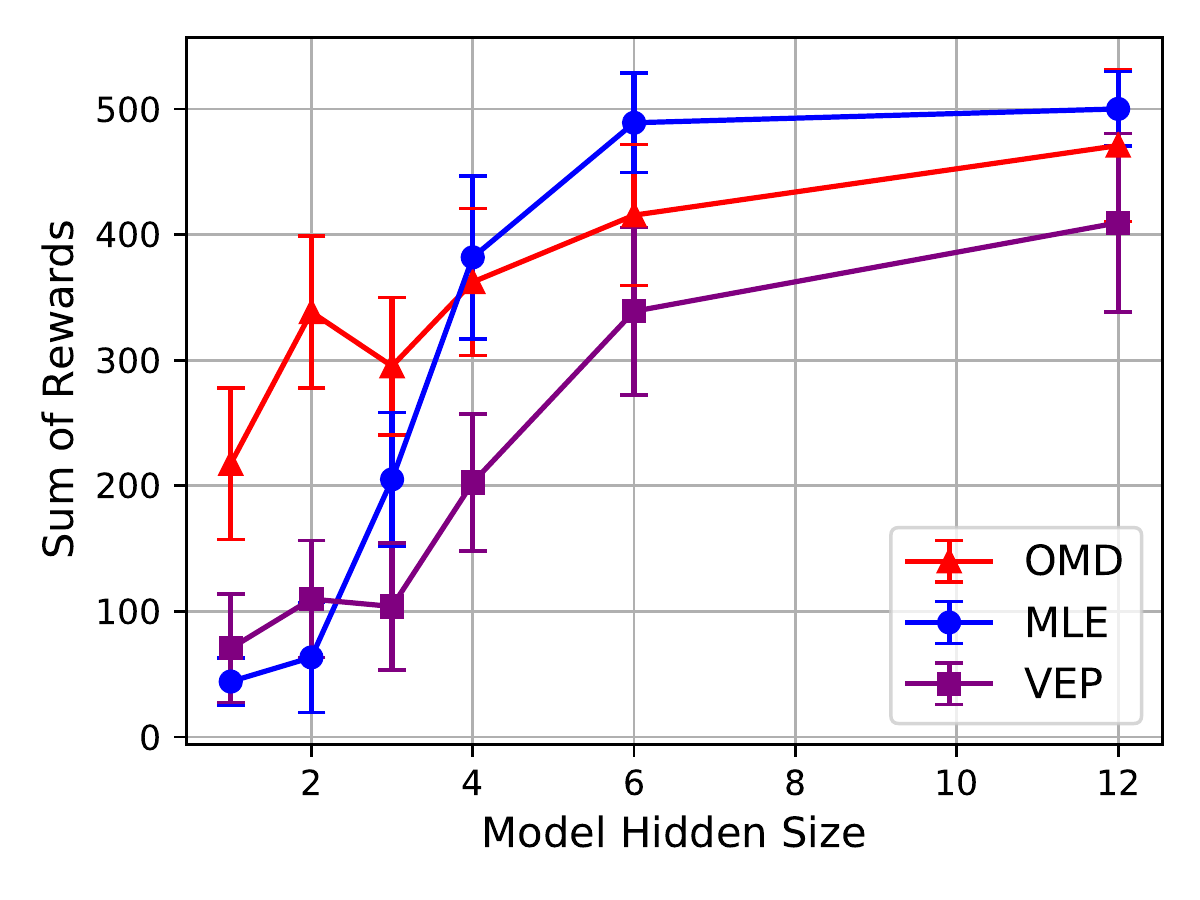}}
\end{subfigure}
~
\begin{subfigure}{0.48\textwidth}
	\centerline{\includegraphics[width=\textwidth]{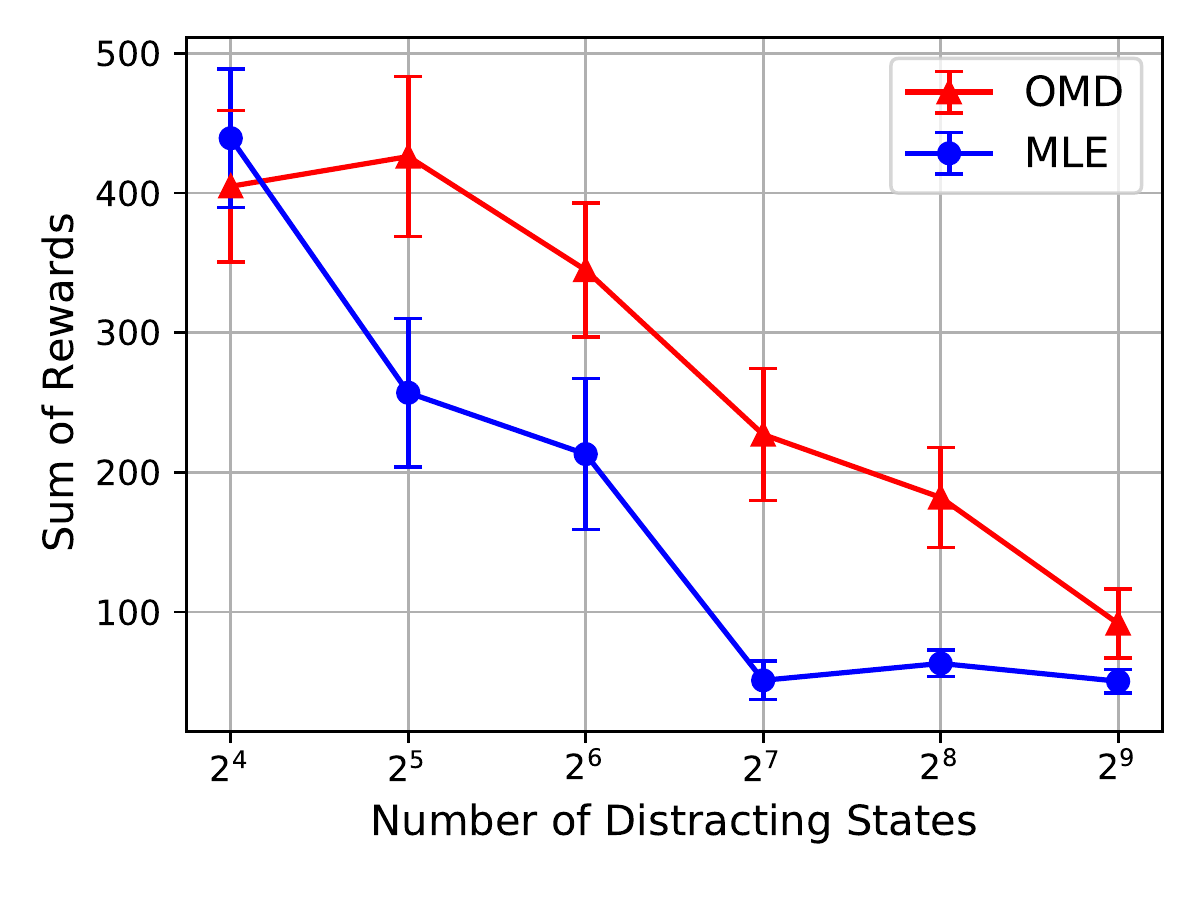}}
\end{subfigure}
\vskip -0.1in
\caption{\textbf{Left:} The final performance of the agents in CartPole for varying hidden dimensionality of the model networks. The OMD model makes useful predictions under the model misspecification. 
\textbf{Right:} The performance when the state space is augmented with uninformative noise. OMD is more robust as the number of distractor increases, while VEP fails for any positive number of distractors. The error bar is the standard error measured over 10 random seeds.}
\label{fig:misspec_approx}
\vskip -0.1in
\end{figure}

\section{Experiments with Function Approximation} 
\label{sec:exp}

This section aims to test the following hypotheses:
\begin{itemize}
    \item The OMD agent with approximations from Section~\ref{sec:omd_cont} achieves near-optimal returns.
    \item The performance of OMD is better compared to MLE under the model misspecification.
    \item The parameters $\theta$ of the OMD model have low likelihood, yet the agent that acts using the Q-function trained with the model achieves near-optimal returns in the true MDP.
\end{itemize}

\textbf{Setup}. 
We provide full details about the experimental setup and hyperparameters in Appendix~\ref{sec:setup}. 
We choose CartPole~\citep{barto1983neuronlike} to have controllable experiments but also include results on MuJoCo HalfCheetah~\citep{todorov2012mujoco} with similar findings in Appendix~\ref{sec:mujoco} further supporting our conclusions. 
Since OMD learns one of the $Q^*$-equivalent models as shown in Section~\ref{sec:analysis}, a close non-MLE baseline would be the algorithm used in the value equivalence principle~(VEP) paper~\citep{grimm2020value}. 
The VEP model minimizes the difference between the Bellman operators:
\begin{equation}
    \label{eq:vep_loss}
    \ell_{\text{VEP}}(\theta) = \sum_{\pi \in \Pi}\sum_{V \in \mathcal{V}}\sum_{s \in \mathcal{S}}\left( B_\pi V(s) - B_\pi^\theta V(s) \right)^2,
\end{equation}
where $B_\pi^\theta V(s) = \mathbb{E}_{a \sim \pi(a|s), s' \sim p_\theta(s'|s,a)} \left(r_\theta(s,a) + \gamma V(s') \right)$, $B_\pi$ is the real model counterpart estimated from samples, and $\Pi$ and $\mathcal{V}$ are predefined sets of policies and value functions.

\textbf{Performance under model misspecification}. 
We design two experiments that allow measuring the misspecification in isolation. First, we limit the model class representational capacity by controlling the number of units in hidden layers of the model. Next, we add distracting states by sampling noise from a standard gaussian and vary the number of distractors. 
Figure~\ref{fig:misspec_approx} shows the returns achieved by the agents after training in the two regimes. 
Note that the Q-function is updated using only the next states and rewards produced by the model, and even when the hidden dimensionality of the model is~1, the OMD model encodes useful information for taking optimal actions. 
Returns achieved by OMD are also more robust to the distractors indicating that the MLE focuses on predicting the parts of a state that might not be relevant for decision making.
The relatively poor performance of VEP suggests that learning a model to predict values for a fixed set of policies and value functions is not as effective, especially if some states are non-informative. 
The experiments reflect the challenges an MBRL agent will face in complex domains such as~\citep{bellemare2013arcade, beattie2016deepmind, kalashnikov2021mt}: accurately predicting the next observations can be infeasible because the underlying dynamics can be too involved and there might be few components that are important for taking action. 
Figure~\ref{fig:misspec_approx} provides evidence that using control-oriented methods would allow using the model capacity more effectively.

\begin{figure}[t]
\centering
\begin{subfigure}{0.32\textwidth}
	\centerline{\includegraphics[width=\textwidth]{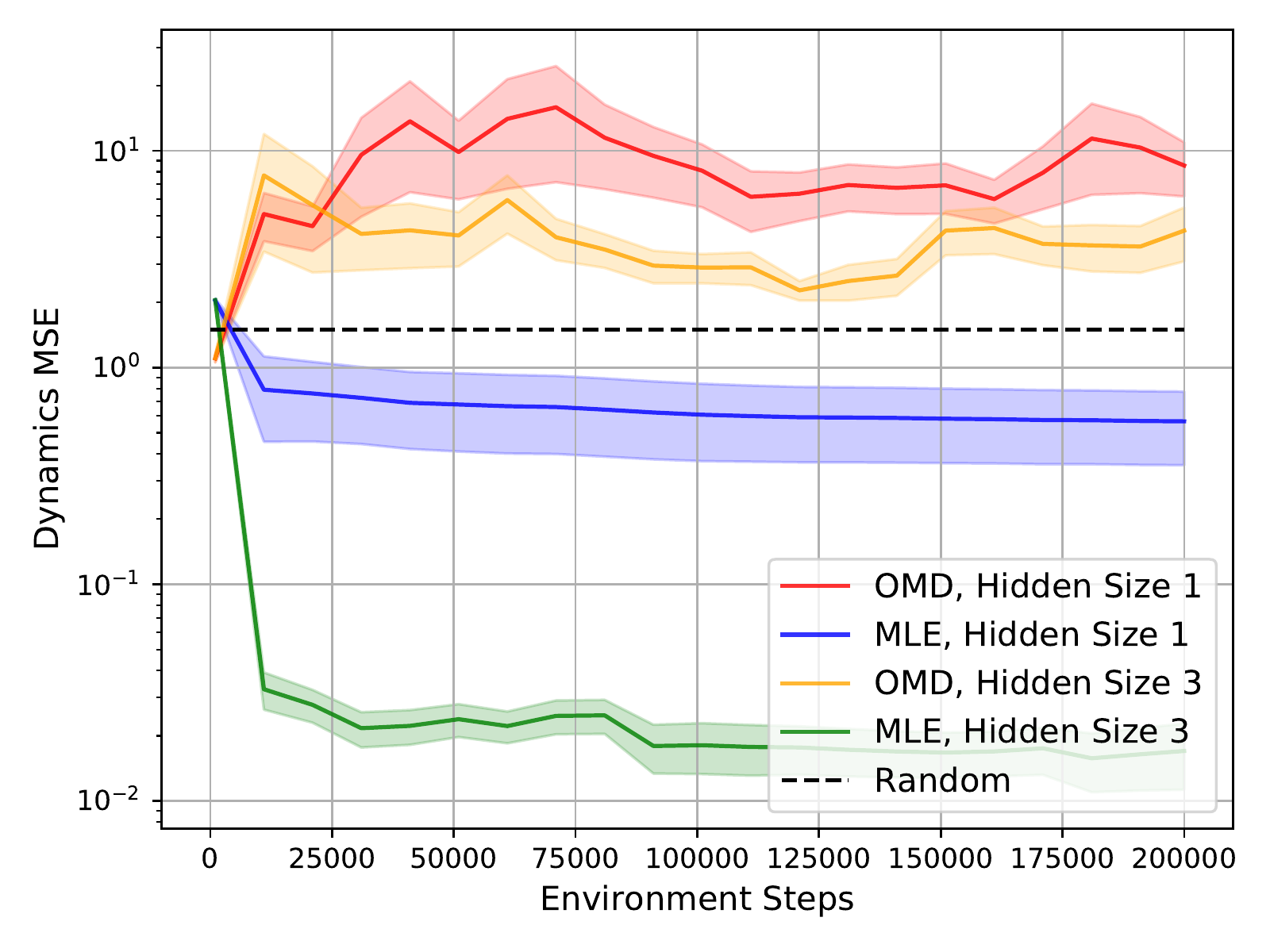}}
\end{subfigure}
~
\begin{subfigure}{0.32\textwidth}
	\centerline{\includegraphics[width=\textwidth]{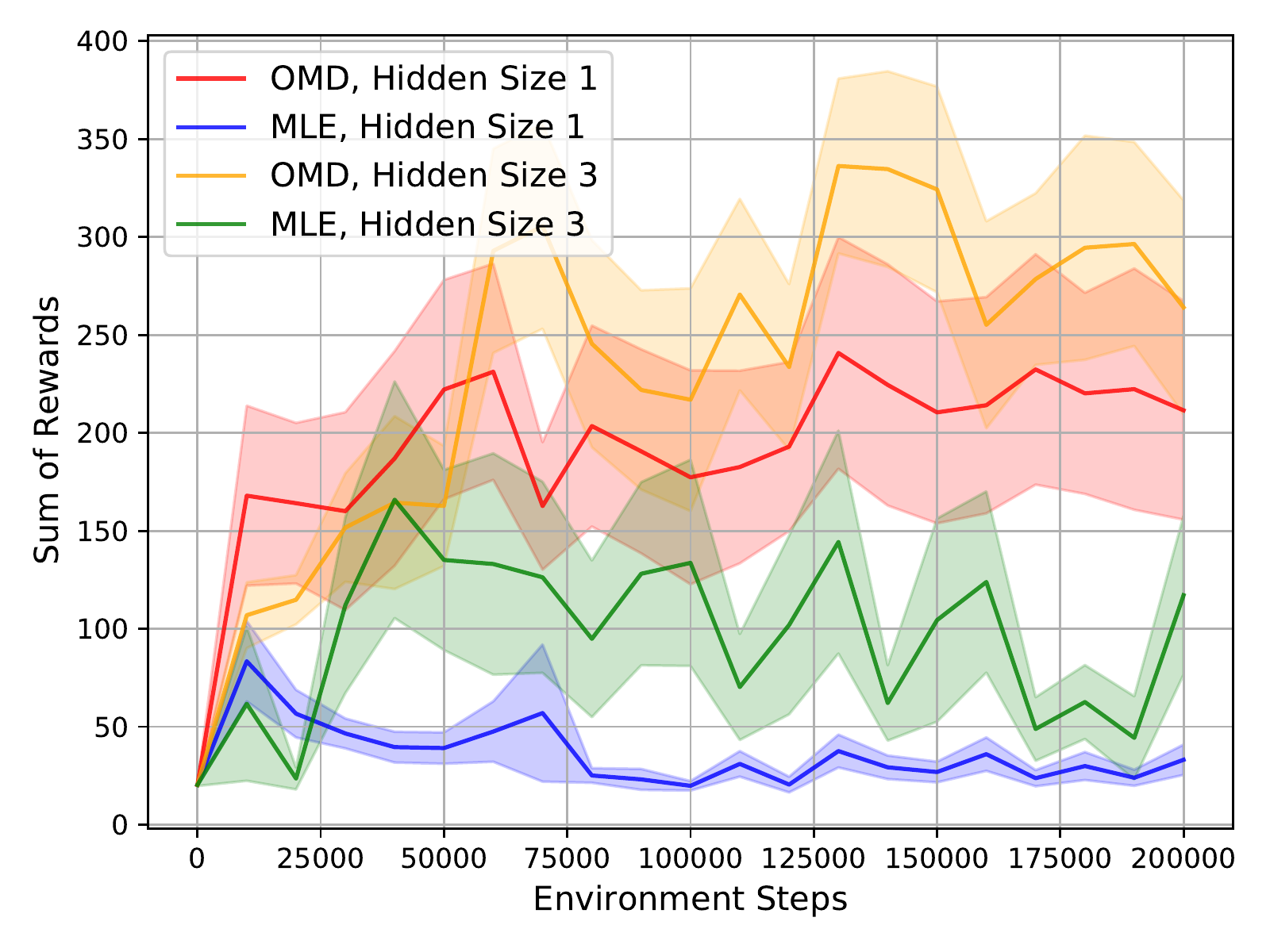}}
\end{subfigure}
~
\begin{subfigure}{0.32\textwidth}
	\centerline{\includegraphics[width=\textwidth]{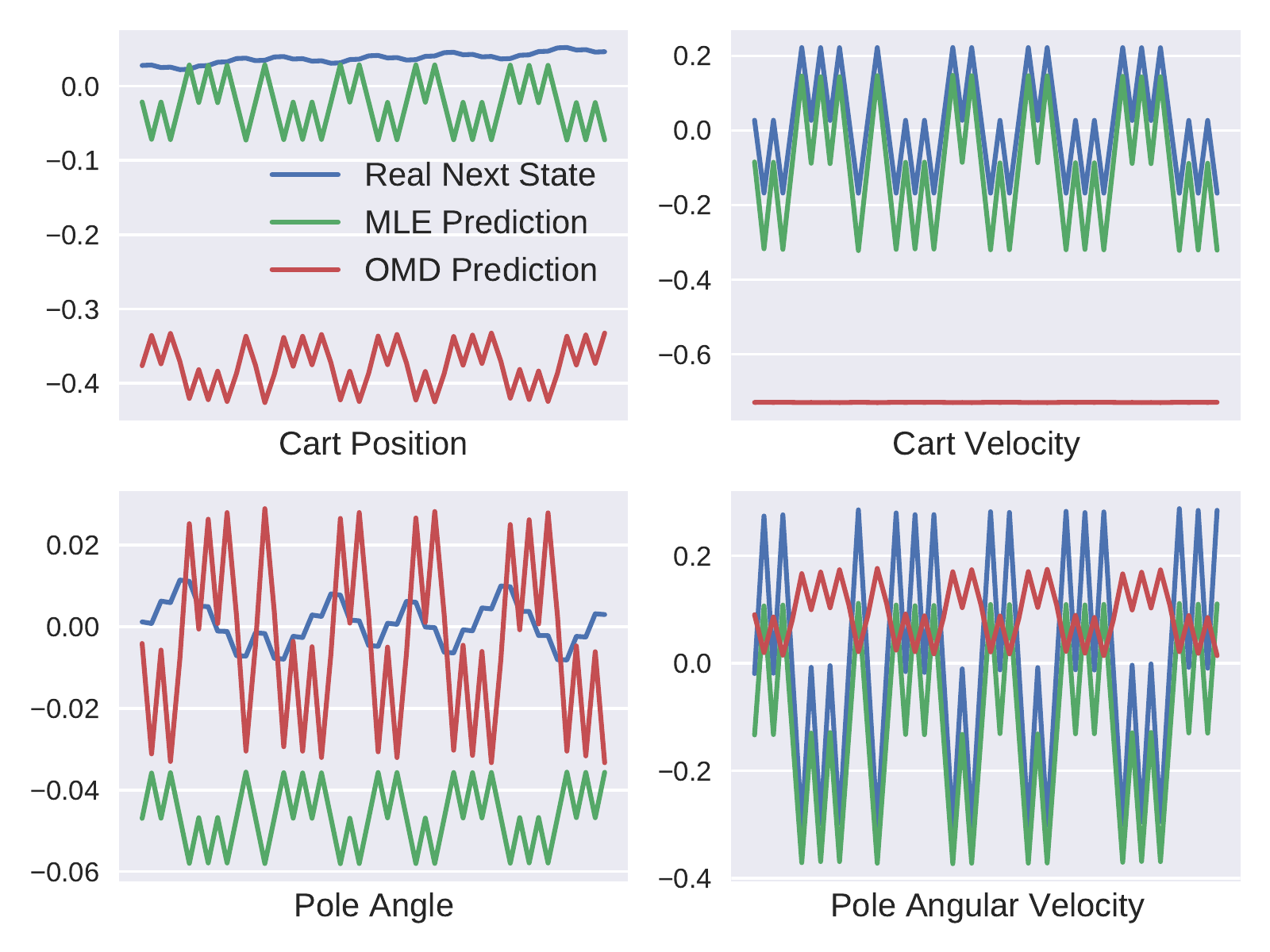}}
\end{subfigure}

\caption{\textbf{Left:} Mean squared error~(MSE) of the next state prediction of OMD and MLE models in CartPole under model misspecification~(number of hidden units 1 and 3). The shaded region is the standard error over 10 runs. 
\textbf{Center:} The corresponding returns of OMD and MLE agents. OMD returns are larger than the returns of MLE even though MLE predicts the next states more accurately. 
\textbf{Right:} Next state predictions of OMD and MLE models with hidden size 1. The OMD agent discards Cart Velocity and predicts unrealistic Cart Position but achieves the optimal returns of 500.}
\label{fig:mse_and_rollouts}
\vskip -0.1in
\end{figure}

\textbf{Likelihood of OMD model.}
We show the mean squared error (MSE) of OMD and MLE dynamics predictions in Figure \ref{fig:mse_and_rollouts}.
Quantitatively, the MSE of OMD models is \emph{higher than the MSE of a randomly initialized model}, while OMD achieves higher returns than MLE.
This finding suggests that the dynamics might not need to produce predictions close to the true states to be useful for planning.

Qualitatively, we visualize individual coordinates of the predictions under the model misspecification in Figure \ref{fig:mse_and_rollouts}. We predict the immediate next states given a sequence of states and actions from an optimal trajectory. We use one of the runs of the OMD agent with the hidden size 1 that achieves optimal returns of 500 and observe that Cart Velocity predictions are nearly constant. On the other hand, an MLE model with the hidden size 1 spends its limited capacity to predict the fluctuations in Cart Velocity leading to a significant deterioration in obtained returns. 

Overall, these findings suggest that the OMD agent achieves near-optimal returns, performs better than the MLE-based MBRL agent as well as VEP under model misspecification, and learns a model that is useful for control despite having low likelihood.

\section{Discussion and Future Work}
\label{sec:future}

An exciting direction for future work is the extension of OMD to environments with image-based observations where model misspecification naturally arises. 
We expect that for complex visual domains, learning a control-oriented model should be more effective compared to model-based methods that rely on reconstruction~\citep{kaiser2019model, hafner2020mastering}.
Based on Figure~\ref{fig:mse_and_rollouts}, we also conjecture that OMD can learn an abstract model that can ignore parts of the original state space that are irrelevant for control. This might allow applying OMD in a zero-shot manner for transfer learning tasks where the underlying dynamics remain unchanged.

Implicit differentiation is not the only way to solve the described constrained optimization problems. Other alternatives include using Lagrangian methods as proposed for the tabular case in \citep{baconlagrangian}. 
Since extending the approach to non-tabular MDPs would require an additional approximator for Lagrange multipliers, we conjecture that finding a saddle point
is going to be less stable than using the IFT. 

Finally, it is worth theoretically studying the sensitivity of the IFT to the approximations to the inverse Jacobian term and the inner loop solutions. 
Our findings, as well as findings of \citep{rajeswaran2019meta, rajeswaran2020game, lorraine2020optimizing} suggest that there is a gap between the assumptions of the IFT and its applicability in practice.

\section{Conclusion}
\label{sec:conclusion}

The paper proposes \emph{optimal model design}~(OMD), a method for learning control-oriented models that addresses the shortcomings of likelihood-based MBRL approaches. OMD optimizes the expected returns in an end-to-end manner and alleviates the objective mismatch of standard MBRL methods that train models using a proxy of the true RL objective. Theoretically, we characterize the set of optimal solutions to OMD and illustrate the efficacy of OMD over MLE agents for approximating optimal value functions. Empirically, we introduce approximations to apply OMD to non-tabular environments and demonstrate the improved performance of OMD in settings with limited model capacity. Perhaps surprisingly, we find that the OMD model can have low likelihood, yet the model is useful for maximizing returns. Overall, OMD sheds light on the potential of control-oriented models for model-based reinforcement learning.

\begin{ack}

EN thanks Iurii Kemaev and Clement Gehring for invaluable help with JAX; Tristan Deleu, Gauthier Gidel, Amy Zhang, Aravind Rajeswaran, Ilya Kostrikov, Brandon Amos, and Aaron Courville for insightful discussions; Pierluca D'Oro, David Brandfonbrener, Valentin Thomas, and Timur Garipov for useful suggestions on the early draft of the paper; Compute Canada for providing computational resources.
This work was partially supported by Facebook CIFAR AI Chair and IVADO.

We acknowledge the Python community \citep{van1995python,oliphant2007python}
for developing the core set of tools that enabled this work, including
JAX \citep{jax2018github},
Jupyter \citep{kluyver2016jupyter},
Matplotlib \citep{hunter2007matplotlib},
numpy \citep{oliphant2006guide,van2011numpy},
pandas \citep{mckinney2012python}, and
SciPy~\citep{jones2014scipy}.

\end{ack}

\bibliographystyle{plainnat}
\bibliography{bibliography}

\newpage
\appendix

{\Large \bf Supplementary Material}

\section{Implicit Differentiation in JAX}
\label{sec:jax_ift}

We provide an implementation of implicit differentiation in JAX \cite{deepmind2020jax, flax2020github} by adapting the code from the library for finding fixed points \cite{gehring2019fax}.
The implementation requires \emph{a solver} that takes parameters $\theta$ and an initial $w_0$ as input and outputs $\varphi(\theta) = w^*$ such that $f(\theta, w^*) = 0$. 
Then, we define a custom vector-Jacobian product that allows using $\varphi(\theta)$ as a block in a differentiable computational graph. 
We highlight the importance of having an implementation \emph{without} explicitly forming the matrices in \eqref{eq:ift} which is crucial for large-scale applications. 
The implementation allows using both the version with the inverse Jacobian term as well as with the identity approximation as discussed in Section~\ref{sec:omd_cont}.

\label{sec:code}
\lstset{
    language=Python,
    basicstyle=\footnotesize\ttfamily,
    keywordstyle=\color{blue}\ttfamily,
    stringstyle=\color{red}\ttfamily,
    commentstyle=\color{green}\ttfamily,
    morecomment=[l][\color{magenta}]{\#}
}

\begin{lstlisting}
import jax.numpy as jnp
from jax import custom_vjp, vjp
from functools import partial
from jax.scipy.sparse.linalg import cg

@partial(custom_vjp, nondiff_argnums=(0, 3))
def root_solve(f, w0, p, solver):
  return solver(f, w0, p)

def fwd(f, w0, p, solver):
  sol = root_solve(f, w0, p, solver)
  return sol, (sol, p)

def rev(f, solver, res, g):
  sol, p = res
  _, dp_vjp = vjp(lambda y: f(y, sol), p)
  if USE_IDENTITY_INVERSE:
    vdp = dp_vjp(-g)[0]
  else:
    _, dsol_vjp = vjp(lambda w: f(p, w), sol)
    vdsoli = cg(lambda v: dsol_vjp(v)[0], g)
    vdp = dp_vjp(-vdsoli[0])[0]
  return jnp.zeros_like(sol), vdp

root_solve.defvjp(fwd, rev)
sol = root_solve(f, w0, p, solver)
# solver returns sol: f(p, sol) = 0
# sol is differentiable w.r.t. p
\end{lstlisting}

The implicit function theorem~(IFT), which we prove in the next appendix, allows differentiating outputs of black-box functions $\varphi(\theta)$ that do not have a closed (differentiable) form.
For example, in Section~\ref{sec:misspec_tab}, function $\varphi$ takes as input the parameters of the model and applies fixed-point iteration until convergence to find the optimal value function for the given parameters.
This contrasts implicit differentiation to other meta-learning algorithms like MAML~\citep{finn2017model} that differentiate \emph{through} the iterations of the inner loop procedure.

\section{Proof of IFT}
\label{sec:proof}

The implicit function theorem is a well known result discussed in, for example, \citep{krantz2012implicit}.
{\renewcommand\footnote[1]{}\ift*}
\begin{proof}
    By the assumption, we have
    \[
        f(\theta, \varphi(\theta)) = \zero \quad \forall \theta \in U.
    \]
    Taking the total derivative of $f$ with respect to $\theta$, we have
    \[
        \frac{\partial f(\theta, w^*)}{\partial \theta} + \frac{\partial f(\theta, w^*)}{\partial w} \frac{\partial\varphi(\theta)}{\partial\theta} \Big\rvert_{w^* = \varphi(\theta)} = \zero.
    \]
    Rearranging the terms and using the invertibility of the Jacobian, we get
    \[
        \frac{\partial\varphi(\theta)}{\partial\theta} = -\left(\frac{\partial f(\theta, w^*)}{\partial w}\right)^{-1} \cdot \frac{\partial f(\theta, w^*)}{\partial \theta} \Big\rvert_{w^* = \varphi(\theta)}.
    \]
\end{proof}

\begin{figure}[t]
\begin{center}
\centerline{\includegraphics[width=0.6\columnwidth]{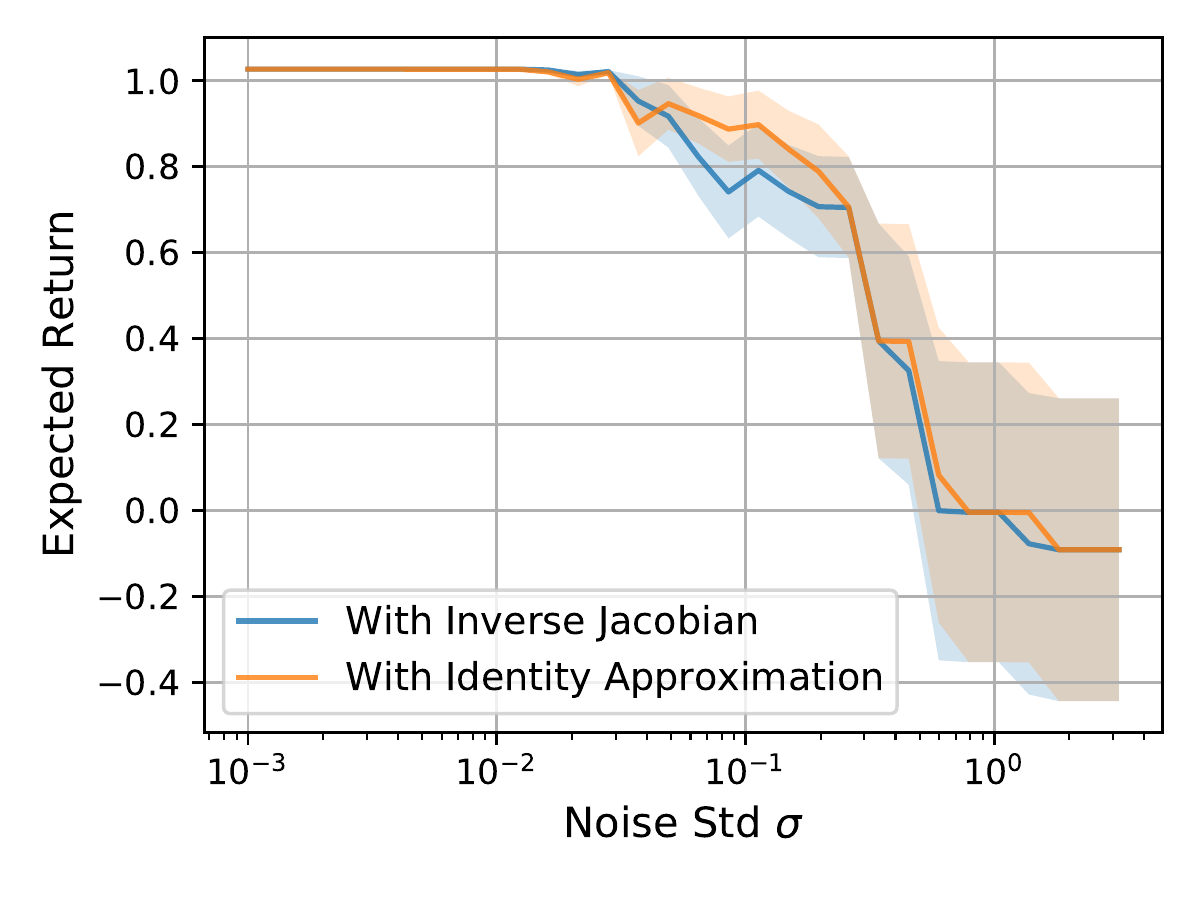}}
\vskip -0.1in
\caption{The expected returns as a function of the inner loop solutions noise magnitude $\sigma$. OMD with the true and the approximate $\frac{\partial \varphi(\theta)}{\partial \theta}$ has the same resulting returns as $\sigma$ increases. The shaded region is the standard error over 10 runs.}
\label{fig:ift_approx_tab}
\end{center}
\vskip -0.2in
\end{figure}

\section{Sensitivity to IFT Approximations}
\label{sec:ift_sens}

The IFT provides a way to calculate the derivatives of a black-box function $\varphi$. 
The expression~\eqref{eq:ift} is valid under the assumptions that
\begin{enumerate}
    \item the inner loop solution $w^*$ satisfies the equation $f(\theta, w^*) = 0$ exactly;
    \item the Jacobian term $\frac{\partial f(\theta, w^*)}{\partial w}$ is inverted accurately.
\end{enumerate}

Ensuring both of the conditions can be challenging for large-scale applications.
This appendix analyzes the sensitivity to the conditions in a controlled manner for the 2-state MDP from Figure~\ref{fig:equiv_mdps_and_bounds}, while Appendix~\ref{sec:ablation} studies the sensitivity for the function approximation case.
At every outer loop iteration, we calculate the exact $w^*$, add gaussian noise with standard deviation $\sigma$, and observe the effect on the expected returns $J$ after the convergence of $\theta$. 
Figure~\ref{fig:ift_approx_tab} demonstrates the results for the exact outer loop gradient $\frac{\partial\varphi(\theta)}{\partial\theta}$ as well as the gradient using the identity approximation of the inverse Jacobian term.
Surprisingly, we did not observe significant benefits of using the Jacobian~$\frac{\partial f(\theta, w^*)}{\partial w}$.
We conjecture that the inverse Jacobian acts like a preconditioner~\citep{boyd2004convex} on $\frac{\partial\varphi(\theta)}{\partial\theta}$ and the preconditioner is useful in our setting only near the exact inner loop solutions $w^*$.
We leave the theoretical investigation of the IFT sensitivity as future work and refer the reader to~\citet{lorraine2020optimizing} for a discussion about approximations of the Jacobian term.

\section{Experimental Details}
\label{sec:setup}

In Section~\ref{sec:exp}, we use CartPole~\citep{barto1983neuronlike}, an environment with 2 actions, 4-dimensional continuous state space, and optimal returns of 500. 
We train the agents for 200000 environment steps. 
The temperature $\alpha$ is 0.01. 
We sample from the replay buffer with a mini-batch size of 256. 
The discount factor $\gamma$ is 0.99. 
At each time step during training, the agent chooses a random action with a probability of 0.1 for exploration. 
We have a separate copy of the environment where we evaluate the agent and take the average over 10 runs to estimate the returns. 
We run each experiment using 10 random seeds.

We set the number of Q-function updates $K$ equal to 1. 
For the results with $K=3$ and $K=10$, see Figure~\ref{fig:ablations}. 
We highlight that after each outer loop step, weights $w$ are warm-started using the last iterate of the previous inner loop (instead of randomly initializing $w$ and training from scratch).
We use Adam optimizer with the learning rate $10^{-3}$ for updating $\theta$ and perform a hyperparameter sweep over the learning rate for $w$ in $\{3 \cdot 10^{-4}, 10^{-3}, 3 \cdot 10^{-3}\}$. 
We make a sweep over the moving average coefficient $\tau$ for the target network $\bar{w}$ in $\{0.005, 0.01\}$. 
Both of the parameters control how fast the Q-network parameters are updated relatively to the model parameters.
Since the CartPole environment is non-stochastic, we use a deterministic dynamics model. 
All networks have two hidden layers and ReLU activations~\citep{nair2010rectified}.
For both hidden layers in all networks, we set the dimensionality to 32. 
In the experiment with the limited model class capacity, we vary the hidden dimensionality in $\{1, 2, 3, 4, 6, 12\}$ for the dynamics and reward networks to measure how the limitation affects the agent's performance.
In the experiment with the distractors, we vary the number of gaussians in $\{2^4, 2^5, 2^6, 2^7, 2^8, 2^9\}$ to measure how the uninformative state components affect the returns.

For the value equivalence principle~(VEP) baseline, we have followed the experimental setup from the original paper~\citep{grimm2020value}.
Perhaps surprisingly, the authors find that it suffices to use a set of all deterministic state-independent policies as $\Pi$ and 5 random value functions as $\mathcal{V}$ for CartPole (see Appendix A.2.3 in~\citep{grimm2020value}). 

We have used CPU-only nodes of the internal cluster.
Each experiment requires 10 seeds $\times$ 3 algorithms $\times$ 2 $\tau$'s $\times$ 6 hidden sizes / 6 numbers of distractors $\times$ 3 LRs resulting in 2160 total jobs.

\begin{figure}[t]
\begin{center}
\begin{subfigure}{0.48\textwidth}
	\centerline{\includegraphics[width=\textwidth]{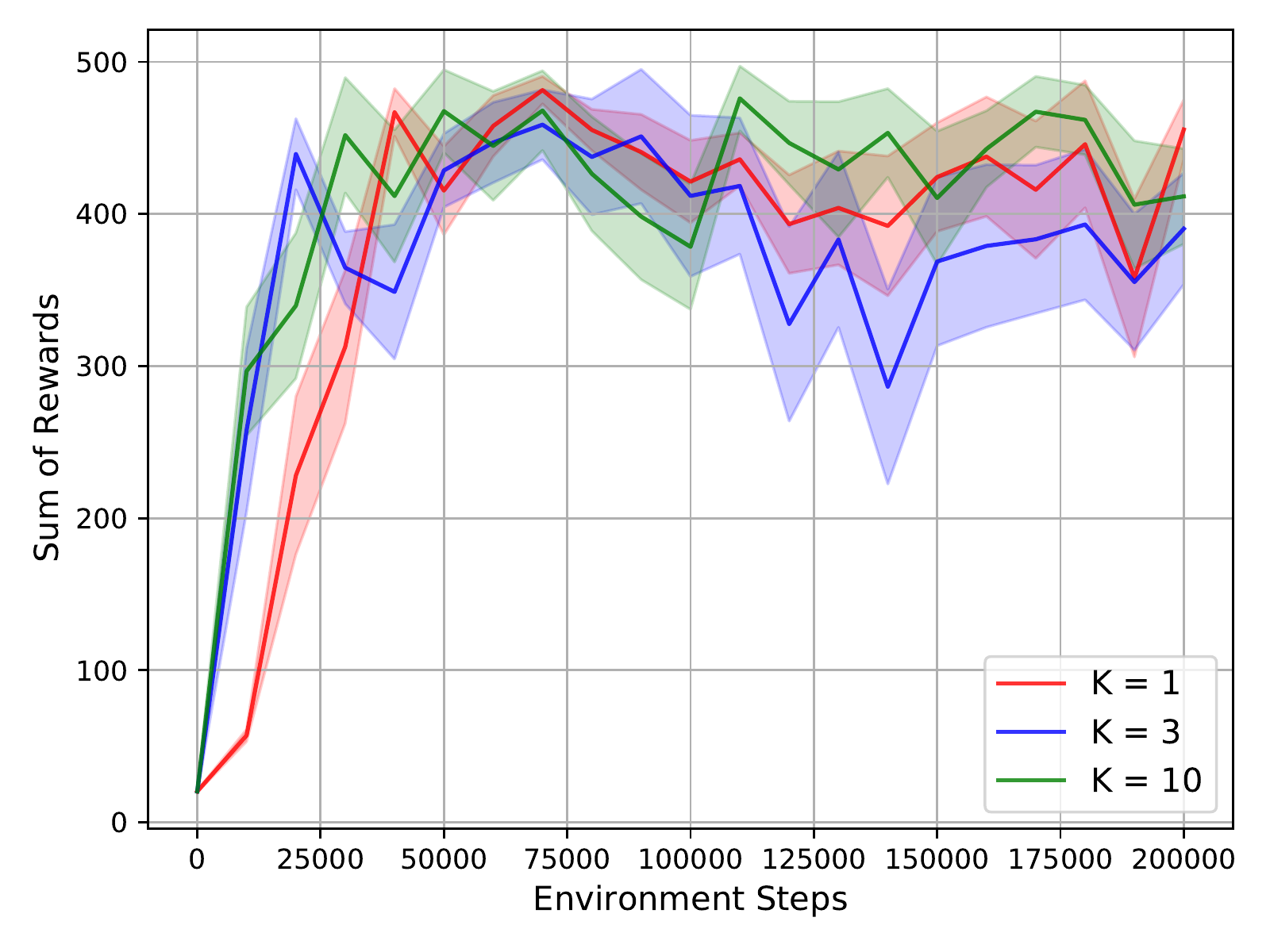}}
\end{subfigure}
~
\begin{subfigure}{0.48\textwidth}
	\centerline{\includegraphics[width=\textwidth]{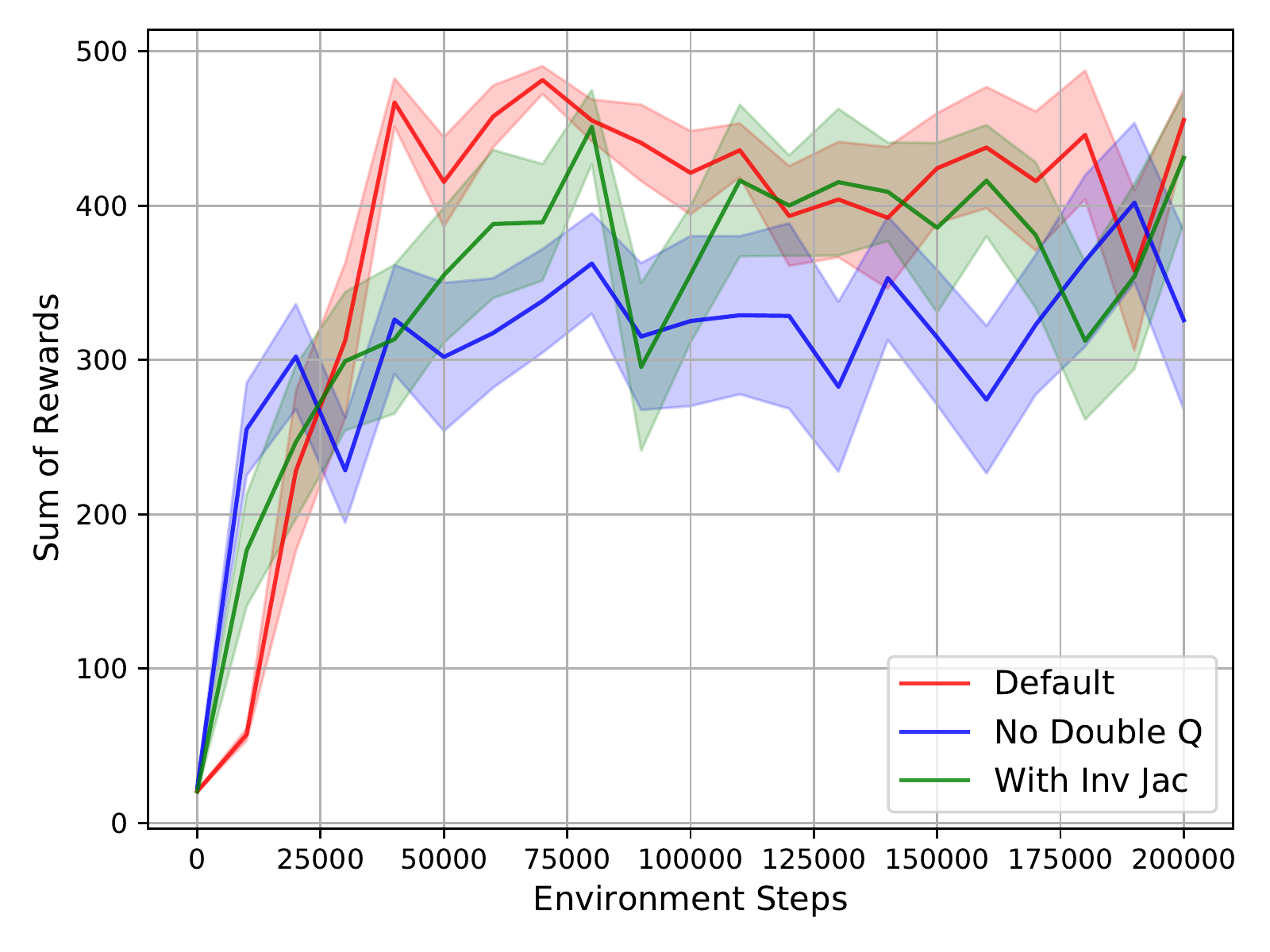}}
\end{subfigure}

\caption{\textbf{Left:} Returns of OMD agents for the varying number of inner loop steps $K$ per outer loop step. The difference between 1, 3, and 10 steps is insignificant.
\textbf{Right:} The evaluation returns of OMD agents for the default agent, the agent without double Q learning, and the agent without the identity approximation of the inverse Jacobian. Using two Q networks increases the returns while using the inverse Jacobian does not change the performance significantly.
The shaded region is the standard error over 10 runs.}
\label{fig:ablations}
\end{center}
\vskip -0.2in
\end{figure}

\section{Ablation Study}
\label{sec:ablation}

Section~\ref{sec:omd_cont} introduces a series of approximations to scale the OMD algorithm to non-tabular environments.
We analyze the effect of the approximations by varying the number of inner loop steps $K$, using the inverse Jacobian term, and using a single Q-function estimator (without double Q-learning).
Figure~\ref{fig:ablations} summarizes the findings of the ablations.
We did not observe significant changes in performance for different values of $K$ for CartPole.
The result suggests that as long as the Q-network update speed (which is also controlled by the target update coefficient $\tau$ and learning rates) stays aligned with the model update speed, adding more inner loop steps is not necessary.

We observed that Q-functions trained with OMD can be prone to overestimation of Q-values showing that double Q-learning is important for OMD.
We conjecture that training a model that maximizes the returns can amplify the overestimation bias~\citep{hasselt2010double} caused by using (soft) maximized sampled targets.

Surprisingly, we did not observe significant benefits from using the inverse Jacobian~$\left(\frac{\partial f(\theta, w^*)}{\partial w}\right)^{-1}$.
We conjecture that there are two reasons explaining the phenomenon.
First, the modified constraint in~\eqref{eq:omd_final} forces the gradient of $L(\theta, w)$ to be zero, implying that the Jacobian is in fact the Hessian matrix $\frac{\partial^2 L(\theta, w^*)}{\partial w^2}$.
\citet{dauphin2014identifying, sagun2017empirical} observed that Hessians of neural networks tend to be singular. 
Since a system of linear equations with a singular matrix has multiple solutions, it is up to a linear algebra solver to choose the solution. 
One of the alternatives would be a min-norm solution corresponding to the Moore-Penrose pseudoinverse and the solution might not be providing a useful inductive bias for the learning process of $\theta$.
Second, the Jacobian term could be useful only in proximity to the exact inner loop solution $w^*$.
Since the practical algorithm performs only $K$ inner loop steps and does not reach the exact $w^*$, the curvature information provided by the Jacobian might not be beneficial for training $\theta$.

Finally, we tried to use the gradient constraint in~\eqref{eq:omd_final} in the tabular setting. 
We got similar results as with the constraint on Q-values in~\eqref{eq:constr} suggesting that the two constraints have similar effects on the learning process.
Overall, the ablation study provides evidence that OMD is robust to the choice of the number of inner loop steps and the IFT approximations, while double Q-learning is the only important algorithmic modification.

\begin{figure}[t]
\begin{center}
\begin{subfigure}{0.48\textwidth}
	\centerline{\includegraphics[width=\textwidth]{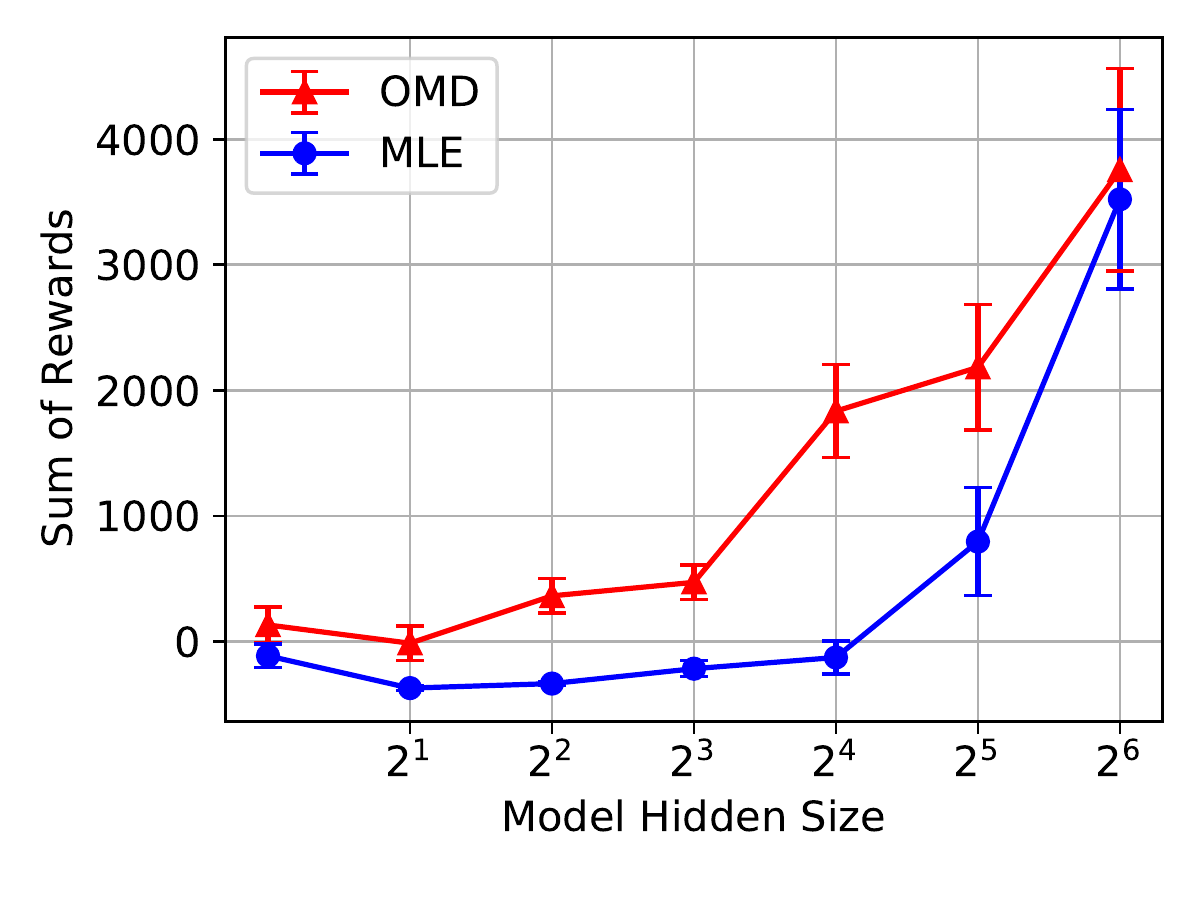}}
\end{subfigure}
~
\begin{subfigure}{0.48\textwidth}
	\centerline{\includegraphics[width=\textwidth]{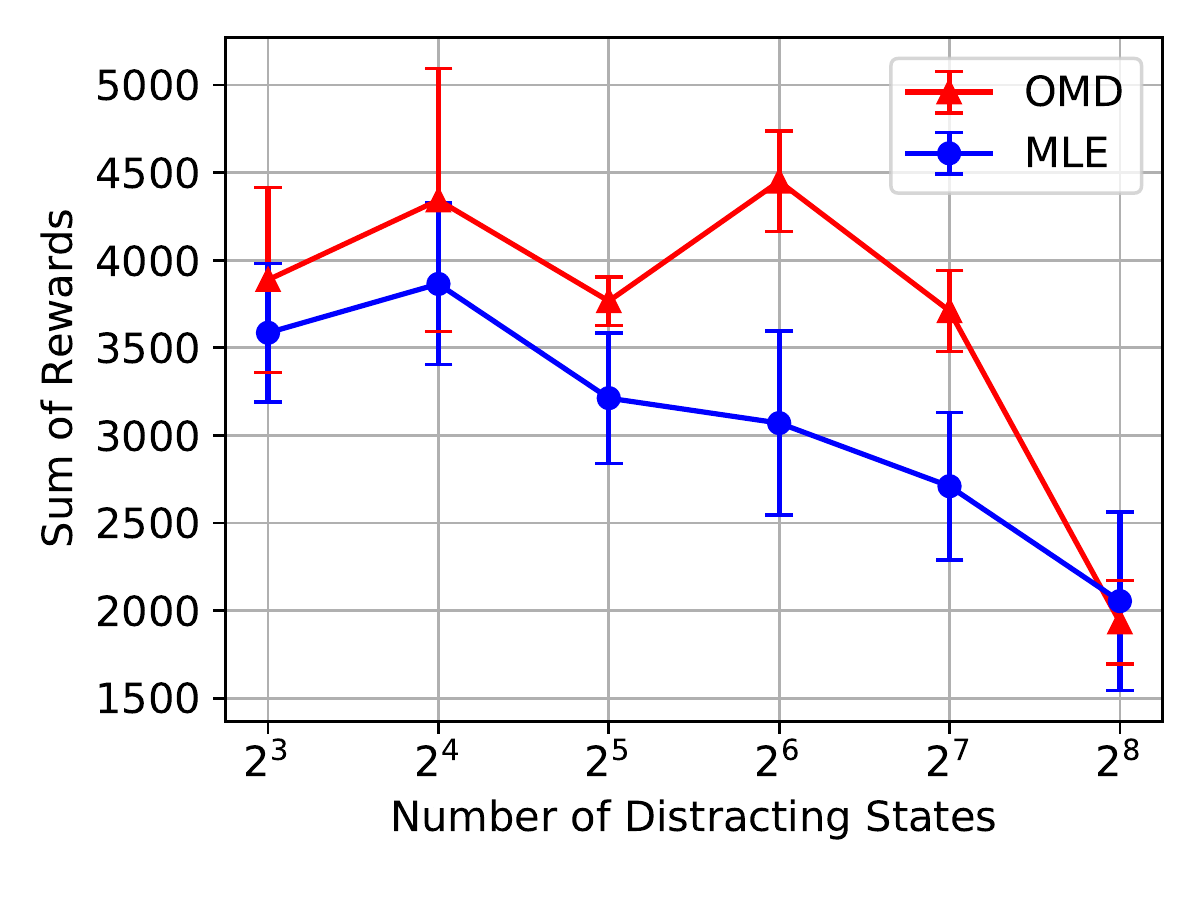}}
\end{subfigure}

\caption{Comparison of OMD and MLE under the model misspecification on HalfCheetah-v2. 
\textbf{Left:}~Returns for varying representational capacity of the model.
\textbf{Right:}~Returns when the state space is augmented with uninformative noise.
In both settings, the OMD model makes more useful predictions.
The std is measured over 5 runs.}
\label{fig:mujoco}
\end{center}
\vskip -0.2in
\end{figure}

\section{Results on HalfCheetah}
\label{sec:mujoco}

We provide an additional comparison of OMD and MLE agents under the model misspecification on MuJoCo HalfCheetah~\citep{todorov2012mujoco}. 
For both agents, the inner optimizer is Soft Actor-Critic~(SAC)~\citep{haarnoja2018soft1, haarnoja2018soft2} with the default configuration.
OMD trains the model using~\ref{eq:grad_final}.
The MLE agent trains the model with MSE effectively becoming the MBPO algorithm~\citep{janner2019trust} without having an ensemble of models and learning the variance of the predictions. 

We perform a hyperparameter sweep over the model learning rate in $\{10^{-4}, 3 \cdot 10^{-4}\}$, over the SAC networks learning rate in $\{10^{-4}, 3 \cdot 10^{-4}\}$, and over $K$ in $\{1, 3\}$. Model hidden size is 64 for the experiment with distractors.

Figure~\ref{fig:mujoco} summarizes the results. Similarly to the observations on the tabular and CartPole environments, the experiments provide evidence that OMD should be preferred over the likelihood-based agent in the model misspecification setup.

\section{Proof of Bounds}
\label{sec:bounds_proof}

Section~\ref{sec:bounds} discusses the bounds on $Q^*$ approximation error obtained by the MLE and OMD agents.
We first prove a lemma relating the error of the model approximation and the Bellman operator approximation.
We then prove a theorem giving a bound on $Q^*$ error.
For simplicity, we focus on the case with the Bellman error in the true MDP \eqref{eq:bellman_outer} as the objective function and ``hard'' versions of the Bellman optimality operators which are obtained by taking the limit of the log-sum-exp temperature $\alpha \to 0$. 
Note that results for MLE hold for any agent that approximates the reward and dynamics functions, but we call the agent MLE since it is a common choice for model parameters estimation.

\textbf{Notation.} 
We denote $p(\cdot|s,a)$ and $Q(\cdot, a)$ as vectors of transition probabilities and Q-values for all states in $\mathcal{S}$. 
The MLE model is given by $(\hat{p}, \hat{r})$ and the corresponding Bellman optimality operator is denoted as $\hat{B}Q$. To have a distinction between OMD and MLE, we denote OMD parameters as $\hat{\theta}$ and the corresponding operator as $B^{\hat{\theta}} Q$.
$\|f\|_\infty = \sup_{x}|f(x)|$ is the infinity norm of a function $f$. 
$\mathbf{1}$ is a vector of an appropriate size with ones as entries.

\begin{restatable}{lemma}{simlemma} \emph{(Bellman operator error bound)} 
Let $Q$ be an action-value function. 
If the dynamics $\hat{p}$ and the reward $\hat{r}$ have the bounded errors 
$\max_{s, a} \left\|p(\cdot|s,a) - \hat{p}(\cdot|s,a)\right\|_1 = \epsilon_p$ and $\max_{s,a}\left|r(s,a) - \hat{r}(s,a)\right| = \epsilon_r$, 
and the reward function is bounded $r(s,a) \in \left[0, r_{\max}\right] \enspace \forall s,a$, we have
    \begin{equation}
        \label{eq:lemma}
        \left\|BQ - \hat{B}Q\right\|_\infty \le \epsilon_r + \frac{\gamma \epsilon_p \Rmax}{2(1-\gamma)}.
    \end{equation}
\end{restatable}
\begin{proof} 
    Using the derivations similar to the proof of the simulation lemma~\citep{jiang2018notes}, we obtain for any state-action pair $(s,a)$
    \begin{align*}
        & \left|BQ(s,a) - \hat{B}Q(s,a)\right|   \\
        =&\left|\left(r(s,a) + \gamma \sum_{s'}p(s'|s,a)\max_{a'}Q(s',a')\right) - \left(\hat{r}(s,a) + \gamma \sum_{s'}\hat{p}(s'|s,a)\max_{a'}Q(s',a')\right)\right| \\
        \le& \left|r(s,a) - \hat{r}(s,a)\right| + \gamma \left|\sum_{s'} \left(p(s'|s,a) - \hat{p}(s'|s,a)\right) \max_{a'}Q(s',a')\right| \\
        \le& \ \epsilon_r\ + \gamma \left|\sum_{s'} \left(p(s'|s,a) - \hat{p}(s'|s,a)\right) \left(\max_{a'}Q(s',a') - \frac{\Rmax}{2(1-\gamma)}\right)\right| \qquad \because p \text{ and } \hat{p} \text{ are distributions} \\
        \le& \  \epsilon_r\ + \gamma \left\|p(\cdot|s,a) - \hat{p}(\cdot|s,a))\right\|_1 \cdot \left\|\max_{a'}Q(\cdot, a') - \frac{\Rmax}{2(1-\gamma)} \mathbf{1} \right\|_\infty \qquad \because \text{H{\"o}lder's inequality}\\
        \le& \ \epsilon_r + \gamma \epsilon_p \left\|\max_{a'}Q(\cdot,a') - \frac{\Rmax}{2(1-\gamma)}\mathbf{1}\right\|_\infty \\
        \le& \ \epsilon_r + \frac{\gamma \epsilon_p \Rmax}{2(1-\gamma)} \qquad \because 0 \le Q(s', a') \le \frac{\Rmax}{1-\gamma}.
    \end{align*}
    Since the inequalities hold for all state-action pairs, we can take the maximum over $(s,a)$ and obtain
        $$\max_{s,a} \left|BQ(s,a) - \hat{B}Q(s,a)\right| \le \epsilon_r + \frac{\gamma \epsilon_p \Rmax}{2(1-\gamma)}.$$
\end{proof}

\qstarerror*
\begin{proof}

    \textbf{(OMD)}\enspace 
    For all state-action pairs $(s,a)$ we get
    \begin{align*}
        & \left|Q^*(s,a) - \hat{Q}(s,a)\right| \\
        = & \left|BQ^*(s,a) - B^{\hat{\theta}}\hat{Q}(s,a)\right| \qquad \because \text{Fixed\ point}\\
        = & \left|BQ^*(s,a) - B\hat{Q}(s,a) + B\hat{Q}(s,a) - B^{\hat{\theta}}\hat{Q}(s,a)\right| \\
        \leq & \left|BQ^*(s,a) - B\hat{Q}(s,a)\right| + \left|B\hat{Q}(s,a) - B^{\hat{\theta}}\hat{Q}(s,a)\right|  \\
        \leq &\ \epsilon + \left|BQ^*(s,a) - B\hat{Q}(s,a)\right|  \\
        = &\ \epsilon + \left|\left(r(s,a) + \gamma \sum_{s'}p(s'|s,a)\max_{a'}Q^*(s',a')\right) - \left(r(s,a) + \gamma \sum_{s'}p(s'|s,a)\max_{a'}\hat{Q}(s',a')\right)\right|  \\
        = &\ \epsilon + \gamma \left|\sum_{s'}p(s'|s,a)\left(\max_{a'}Q^*(s',a') - \max_{a'}\hat{Q}(s',a')\right)\right| \\
        \leq &\ \epsilon + \gamma \left\|p(\cdot|s,a)\right\|_1 \cdot \left\|\max_{a'}Q^*(\cdot,a') - \max_{a'}\hat{Q}(\cdot,a')\right\|_\infty  \qquad \because \text{H{\"o}lder's inequality} \\
        = &\ \epsilon + \gamma \left\|\max_{a'}Q^*(\cdot,a') - \max_{a'}\hat{Q}(\cdot,a')\right\|_\infty \qquad \because p \text{ is a distribution} \\
        \leq &\ \epsilon + \gamma \max_{a'}\left\|Q^*(\cdot,a') - \hat{Q}(\cdot,a')\right\|_\infty \\
        = &\  \epsilon + \gamma \max_{s',a'}\left|Q^*(s',a') - \hat{Q}(s',a')\right|.
    \end{align*}
    Taking the maximum over $(s,a)$, we get the recursion:
    \begin{align*}
        \max_{s,a}\left|Q^*(s,a) - \hat{Q}(s,a)\right|
        &\leq \epsilon + \gamma \max_{s,a}\left|Q^*(s,a) - \hat{Q}(s,a)\right| \\
        \max_{s,a}\left|Q^*(s,a) - \hat{Q}(s,a)\right|
        &\leq \frac{\epsilon}{(1 - \gamma)}.
    \end{align*}
    
    \textbf{(MLE)}\enspace 
    The proof for MLE can be obtained using the same derivations and additionally using the result of the lemma bounding the difference between the Bellman operators: 
    $$
        \epsilon = \max_{s,a} \left|B\hat{Q}_{\mathrm{MLE}}(s,a) - \hat{B} \hat{Q}_{\mathrm{MLE}}(s,a)\right| \le \epsilon_r + \frac{\gamma \epsilon_p \Rmax}{2(1 - \gamma)}.
    $$
    
\end{proof}

The last inequality demonstrates that the OMD bound is tighter. 
OMD model directly optimizes $\left|B\hat{Q}(s,a) - \hat{Q}(s,a)\right| = \left|B\hat{Q}(s,a) - B^{\hat{\theta}}\hat{Q}(s,a)\right|$, while MLE minimizes $\epsilon_r$ and $\epsilon_p$ that only upper bound $\left|B\hat{Q}(s,a) - \hat{B}\hat{Q}(s,a)\right|$ as suggested by the lemma.
Hence, given the same budget of representational capacity, OMD will learn a model that is more helpful for approximating the optimal Q-function. 
Finally, Figure~\ref{fig:equiv_mdps_and_bounds} empirically supports our theoretical findings showing that both the $Q^*$ approximation error and the error-bound gap are smaller for OMD.

\end{document}